\documentclass{clv3}

\usepackage{hyperref}
\usepackage{xcolor}
\definecolor{darkblue}{rgb}{0, 0, 0.5}
\hypersetup{colorlinks=true,citecolor=darkblue, linkcolor=darkblue, urlcolor=darkblue}

\historydates{Submission received: 8 June 2021; revised version received: 18 August 2021; accepted for publication: 26 August 2021.}

\issue{47}{4}{2021}

\runningtitle{Sequence-Level Training}
\runningauthor{Shao et al.}

\usepackage{booktabs}
\usepackage{bm}
\usepackage{amsmath}
\usepackage{amssymb}
\usepackage{algorithmicx}
\usepackage{algorithm}
\usepackage{algpseudocode}
\usepackage{multirow}

\begin{document}
\title{Sequence-Level Training for Non-Autoregressive Neural Machine Translation}

\author{Chenze Shao}
\affil{Key Laboratory of Intelligent Information Processing\\Institute of Computing Technology\\Chinese Academy of Sciences\\
\texttt{shaochenze18z@ict.ac.cn}}

\author{Yang Feng}
\affil{Key Laboratory of Intelligent Information Processing\\Institute of Computing Technology\\Chinese Academy of Sciences\\
\texttt{fengyang@ict.ac.cn}}

\author{Jinchao Zhang}
\affil{Pattern Recognition Center\\ WeChat AI, Tencent Inc\\
\texttt{dayerzhang@tencent.com}}

\author{Fandong Meng}
\affil{Pattern Recognition Center\\ WeChat AI, Tencent Inc\\
\texttt{fandongmeng@tencent.com}}

\author{Jie Zhou}
\affil{Pattern Recognition Center\\ WeChat AI, Tencent Inc\\
\texttt{withtomzhou@tencent.com}}

\maketitle
\historydates{Submission received: 8 June 2021; revised version received: 18 August 2021; accepted for publication: 26 August 2021.}
\begin{abstract}
In recent years, Neural Machine Translation (NMT) has achieved notable results in various translation tasks. However, the word-by-word generation manner determined by the autoregressive mechanism leads to high translation latency of the NMT and restricts its low-latency applications. Non-Autoregressive Neural Machine Translation (NAT) removes the autoregressive mechanism and achieves significant decoding speedup through generating target words independently and simultaneously. Nevertheless, NAT still takes the word-level cross-entropy loss as the training objective, which is not optimal because the output of NAT cannot be properly evaluated due to the multimodality problem. In this article, we propose using sequence-level training objectives to train NAT models, which evaluate the NAT outputs as a whole and correlates well with the real translation quality. Firstly, we propose training NAT models to optimize sequence-level evaluation metrics (e.g., BLEU) based on several novel reinforcement algorithms customized for NAT, which outperforms the conventional method by reducing the variance of gradient estimation. Secondly, we introduce a novel training objective for NAT models, which aims to minimize the Bag-of-Ngrams (BoN) difference between the model output and the reference sentence. The BoN training objective is differentiable and can be calculated efficiently without doing any approximations. Finally, we apply a three-stage training strategy to combine these two methods to train the NAT model. We validate our approach on four translation tasks (WMT14 En$\leftrightarrow$De, WMT16 En$\leftrightarrow$Ro), which shows that our approach largely outperforms NAT baselines and achieves remarkable performance on all translation tasks. The source code is available at https://github.com/ictnlp/Seq-NAT.
\end{abstract}

\section{Introduction}
Machine translation used to be one of the most challenging tasks in natural language processing, but recent advances in neural machine translation make it possible to translate with an end-to-end model architecture. NMT models are typically built on the encoder-decoder framework. The encoder network encodes the source sentence to distributed representations, and the decoder network reconstructs the target sentence from these representations in an autoregressive manner. The target sentence is generated word-by-word where the previously predicted words are fed back to the decoder as context. In the past few years, autoregressive NMT models have achieved notable results in various translation tasks \cite{cho-etal-2014-learning,sutskever2014sequence,bahdanau2014neural,wu2016google,vaswani2017attention}. However, the word-by-word generation manner determined by the autoregressive mechanism leads to high translation latency of the NMT and restricts its low-latency applications. 

Non-Autoregressive Neural Machine Translation (NAT) \cite{gu2017non} is proposed to reduce the latency of NMT. By removing the autoregressive mechanism, NAT can generate target words independently and simultaneously, thereby achieving significant decoding speedup. Nevertheless, NAT still takes the word-level cross-entropy loss as the training objective, which is not optimal because the output of NAT cannot be properly evaluated. Due to the multimodality of language, the reference sentence may have many variants that are composed of different words but have the same semantics. For the autoregressive model, the teacher forcing algorithm \cite{williams1989learning} can provide it with sequential information that guides the model to generate the reference sentence. However, the sequential information is not available during the training of NAT, so NAT may generate any translation variant with the target semantics. Once the NAT tends to generate a variant that is not aligned verbatim with the reference sentence, the cross-entropy loss will give it a large penalty with no regard to the translation quality. Consequently, the correlation between the cross-entropy loss and translation quality becomes weak, which has a negative impact on the NAT performance.

As shown in Figure \ref{fig:overcorrection}, though the translation ``I have to get up and start working.'' has similar semantics to the reference sentence, the word-level cross-entropy loss will give it a large penalty since it is not aligned verbatim with the reference sentence. Under the guidance of cross-entropy loss, the translation may be further corrected to ``I have to up up start start working.''. This is preferred by the cross-entropy loss but the translation quality will actually get worse, which is named as the overcorrection error~\cite{zhang-etal-2019-bridging}. The essential reason for the overcorrection error is that the loss function evaluates the generation quality of each position independently and does not model the sequential dependency. As a result, NAT tends to focus on local correctness while ignoring the overall translation quality, and therefore generates influent translations with many over-translation and under-translation errors. As shown in Table \ref{tab:case}, the output of NAT is incomplete and contains repeated words like `cancer' and `aggressive'.
\begin{figure}[ht]
  \begin{center}
    \includegraphics[width=.8\columnwidth]{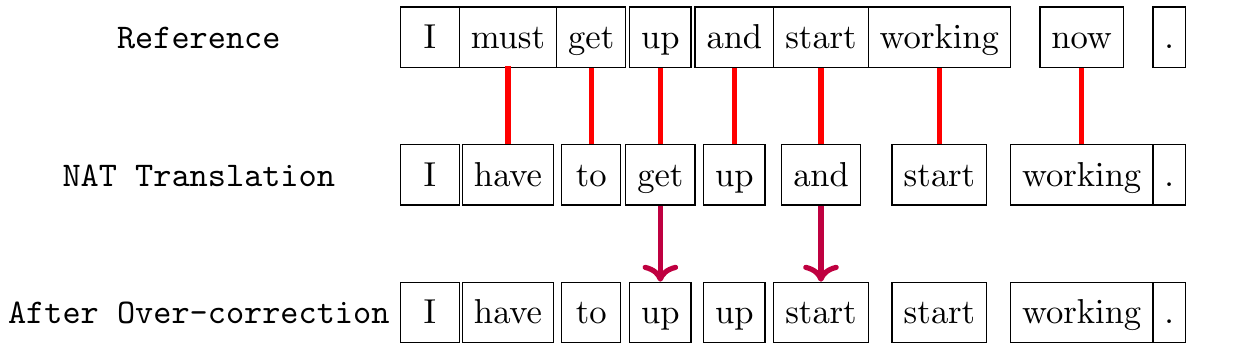}
    \caption{One example of NAT translation that is not aligned verbatim with the reference sentence. The red line indicates the misalignment that will receive a large penalty from the word-level cross-entropy loss. The purple arrow indicates the possible overcorrection error.}
    \label{fig:overcorrection}
  \end{center}
\end{figure}
\begin{table}[ht]
\centering
%\setlength{\tabcolsep}{4pt}
%\resizebox{\textwidth}{!}{
%\vskip 0.1in
\caption{A translation case on the validation set of WMT14 De-En. Source and Target are the source sentence and reference sentence, respectively. AT and NAT are the output of the autoregressive Transformer and non-autoregressive Transformer, respectively.}
\begin{tabular}{c|l}
\toprule
{Source} & Es gibt Krebsarten , die aggressiv und andere , die indolent sind .\\
\hline
{Reference} & There are aggressive cancers and others that are indolent .\\
\hline
{AT} & There are cancers that are aggressive and others that are indolent .\\
\hline
{NAT} & There are cancers cancer aggressive aggressive others are indindent .\\
\bottomrule
\end{tabular}
\label{tab:case}
\end{table}

In this article, we propose using sequence-level training objectives to train NAT models, which evaluate the NAT outputs as a whole and correlate well with the real translation quality. Firstly, we propose training NAT models to optimize sequence-level evaluation metrics (e.g., BLEU \cite{papineni2002bleu}, GLEU \cite{wu2016google}, and ROUGE \cite{lin-2004-rouge}). These metrics are usually non-differentiable, and reinforcement learning techniques \cite{sutton1984temporal,williams1992simple,sutton2000policy} are widely applied to train autoregressive NMT to optimize these discrete objectives \cite{ranzato2015sequence,bahdanau2016actor}. However, the training procedure is usually unstable due to the high variance of the gradient estimation. Using the appealing characteristics of non-autoregressive generation, we propose several novel reinforcement algorithms customized for NAT, which outperforms the conventional method by reducing the variance of gradient estimation. Secondly, we introduce a novel training objective for NAT models, which aims to minimize the Bag-of-Ngrams (BoN) difference between the model output and the reference sentence. As the word-level loss cannot properly model the sequential dependency, we propose to evaluate the NAT output at the n-gram level. Since the output of NAT may not be aligned verbatim with the reference, we do not require the strict alignment and optimize the bag-of-ngrams for NAT. Optimizing such an objective usually faces the difficulty of the exponential search space, and we find that the difficulty can be overcome through using the characteristics of non-autoregressive generation. In summary, the BoN training objective has many appealing properties. It is differentiable and can be calculated efficiently without doing any approximations. Most importantly, the BoN objective correlates well with the overall translation quality, as we demonstrate in the experiments.

The reinforcement learning method can train NAT with any sequence-level objective, but it requires a lot of calculations on the CPU to reduce the variance of gradient estimation. The bag-of-ngrams method can efficiently calculate the BoN objective without doing any approximations, but the choice of training objectives is very limited. The cross-entropy loss also has strengths such as high-speed training and is suitable for model warmup. Therefore, we apply a three-stage training strategy to combine the two sequence-level training methods and the word-level training to train the NAT model. We validate our approach on four translation tasks (WMT14 En$\leftrightarrow$De, WMT16 En$\leftrightarrow$Ro), which shows that our approach largely outperforms NAT baselines and achieves remarkable performance on all translation tasks.

This article extends our conference papers on non-autoregressive translation \cite{shao-etal-2019-retrieving,DBLP:conf/aaai/ShaoZFMZ20} in three major directions. Firstly, we propose several novel sequence-level training algorithms in this article. In the context of reinforcement learning, we propose the traverse-based method Traverse-Ref, which practically eliminates the variance of gradient estimation and largely outperforms the best method Reinforce-Topk proposed in \citet{shao-etal-2019-retrieving}. We also propose to use bag-of-words as the training objective of NAT. The bag-of-words vector can be explicitly calculated, so it supports a variety of distance metrics such as BoW-$L_1$, BoW-$L_2$ and BoW-$Cos$ as loss functions, which enables us to analyze the performance of different distance metrics on NAT. Secondly, we explore the combination of the reinforcement learning based method and the bag-of-ngrams method and propose a three-stage training strategy to better combine their advantages. Finally, we conduct experiments on a stronger baseline model \cite{ghazvininejad2019maskpredict} and a larger batch size setting to show the effectiveness of our approach, and we also provide more detailed analysis. The article is structured as follows. We explain the vanilla non-autoregressive translation and sequence-level training in Section 2. We introduce our sequence-level training methods in Section 3. We review the related works on non-autoregressive translation and sequence-level training in Section 4. In Section 5, we introduce the experimental design, conduct experiments to evaluate the performance of our methods and conduct a series of analyses to understand the underlying key components in them. Finally, we conclude in Section 6 by summarizing the contributions of our work.

\section{Background}
\subsection{Autoregressive Neural Machine Translation}
Deep neural networks with the autoregressive encoder-decoder framework have achieved state-of-the-art results on machine translation, with different choices of architectures such as Recurrent Neural Network (RNN), Convolutional Neural Network (CNN) and Transformer. RNN-based models~\cite{bahdanau2014neural,cho-etal-2014-learning} have a sequential architecture that makes them difficult to be parallelized. CNN~\cite{gehring2017convolutional}, and self-attention~\cite{vaswani2017attention} based models have highly parallelized architectures, which solves the parallelization problem during the training. However, during the inference, the translation has to be generated word-by-word due to the autoregressive mechanism. 

Given a source sentence $\bm{X}=\{x_1, ..., x_{n}\}$ and a target sentence $\bm{Y}=\{y_1, ..., y{_T}\}$, the autoregressive NMT models the translation probability from $\bm{X}$ to $\bm{Y}$ sequentially as:
%an autoregressive neural machine translation model with parameters $\theta$ models the translation probability from $\bm{x}$ to $\bm{y}$ autoregressively:
\begin{equation}
\label{eq:auto_prob}
P(\bm{Y}|\bm{X},\theta) = \prod_{t=1}^{T}p(y_t|\bm{y_{<t}},\bm{X},\theta)
\end{equation}
where $\theta$ is a set of model parameters and $\bm{y_{<t}}=\{y_1,\cdots,y_{t-1}\}$ is the translation history. The standard training objective is the cross-entropy loss, which minimizes the negative log-likelihood as:
\begin{equation}
\begin{aligned}
\label{eq:auto_mle}
\mathcal{L}_{MLE}(\theta) = -\sum_{t=1}^{T}\log(p(y _t|\bm{ y_}{<t},\bm{X},\theta))
\end{aligned}
\end{equation}

During the training, the teacher forcing algorithm~\cite{williams1989learning} is applied, where golden target words are fed into the decoder as the translation history. During the inference, since there is no polynomial time algorithm to find the translation with the maximum likelihood, autoregressive models have to rely on decoding algorithms such as greedy search and beam search to generate the translation. The partial translation generated by the decoding algorithm is fed back to the decoder to guide the generation of the next word. The prominent feature of the autoregressive model is that it requires the target-side historical information in the decoding procedure. Therefore, target words are generated in the word-by-word style, which leads to high translation latency and restricts the application of the autoregressive model.

\subsection{Non-Autoregressive Neural Machine Translation}
Non-autoregressive neural machine translation \cite{gu2017non} is proposed to reduce the translation latency through parallelizing the decoding process. A basic NAT model takes the same encoder-decoder architecture as the Transformer, except that there is a length predictor that takes encoder states as input to predict the target length.
Given a source sentence $\bm{X}=\{x_1, ..., x_{n}\}$ and a target sentence $\bm{Y}=\{y_1, ..., y{_T}\}$, NAT models the translation probability from $\bm{X}$ to $\bm{Y}$ as:

\begin{equation}
\label{eq:nonauto_prob}
P(\bm{Y}|\bm{X},\theta) = \prod_{t=1}^{T}p_t(y_t|\bm{X},\theta)
\end{equation}
where $\theta$ is a set of model parameters and $p_t(y_t|\bm{X},\theta)$ indicates the translation probability of word $y_t$ in position $t$. The cross-entropy loss is applied to minimize the negative log-likelihood as:
\begin{equation}
\begin{aligned}
\label{eq:nonauto_mle}
\mathcal{L}_{MLE}(\theta) = -\sum_{t=1}^{T}\log(p_t(y_t|\bm{X},\theta))
\end{aligned}
\end{equation}

During the training, the target length $T$ is usually set as the reference length. During the inference, the target length $T$ is obtained from the length predictor, and then the translation of length $T$ with the maximum likelihood can be easily obtained by taking the word with the maximum likelihood at each time step:
\begin{equation}
\label{eq:argmax}
\hat{y_t} = \arg\max_{y_t}p_t(y_t|\bm{X},\theta)
\end{equation}

\subsection{Sequence-Level Training for Autoregressive NMT}
Reinforcement learning techniques \cite{sutton1984temporal,Ng1999PolicyIU,sutton2000policy} have been widely applied to improve the performance of the autoregressive NMT with sequence-level objectives \cite{ranzato2015sequence,bahdanau2016actor,wu2016google,yang-etal-2018-improving}. As sequence-level objectives are usually non-differentiable, the loss function is defined as the negative expected reward:
\begin{equation}
\mathrm{L}_\theta=- \sum_{\bm{\mathrm{Y}}=y_{1:T}}{p(\bm{\mathrm{Y}}|\bm{\mathrm{X}},\theta) \cdot r(\bm{\mathrm{Y}})} \label{eq:loss}
\end{equation}
where $\bm{\mathrm{Y}}=y_{1:T}$ denotes a possible sequence generated by the model, and $r(\bm{\mathrm{Y}})$ is the corresponding reward (e.g., BLEU) for generating the sentence $\bm{\mathrm{Y}}$. Enumerating all possible target sentences is impossible due to the exponential search space, and REINFORCE \cite{williams1992simple} gives an elegant way to estimate the gradient:
\begin{equation}
\begin{aligned}
\label{eq:auto_rf}
\nabla_{\theta}\mathrm{L}_\theta=&- \sum_{\bm{\mathrm{Y}}}p(\bm{\mathrm{Y}}|\bm{\mathrm{X}},\theta)\cdot \frac{\nabla_{\theta}p(\bm{\mathrm{Y}}|\bm{\mathrm{X}},\theta)}{p(\bm{\mathrm{Y}}|\bm{\mathrm{X}},\theta)} \cdot r(\bm{\mathrm{Y}})\\
=&- \mathop{\mathbb{E}}\limits_{\bm{\mathrm{Y}}} [\nabla_{\theta} \log(p(\bm{\mathrm{Y}}|\bm{\mathrm{X}},\theta)) \cdot r(\bm{\mathrm{Y}})]\\
=&- \mathop{\mathbb{E}}\limits_{\bm{\mathrm{Y}}} [\sum_{t=1}^{T}\nabla_{\theta} \log(p(y_t|\bm{y_{<t}},\bm{\mathrm{X}},\theta)) \cdot r(\bm{\mathrm{Y}})]
\end{aligned}
\end{equation}

Equation (\ref{eq:auto_rf}) indicates that we can obtain an unbiased estimation of the gradient through the training process summarized as follows. 
\begin{itemize}
\item Given a source sentence $\bm{X}$, sample a translation $\bm{Y}$ from $p(\bm{\mathrm{Y}}|\bm{\mathrm{X}},\theta)$.
\item Calculate the reward $r(\bm{\mathrm{Y}})$ for the sampled sentence.
\item Update the model by the estimated gradient $\nabla_{\theta} \log(p(\bm{\mathrm{Y}}|\bm{\mathrm{X}},\theta)) \cdot r(\bm{\mathrm{Y}})$
\end{itemize}

Although it has the advantage of unbiased estimation, previous investigations show that the reinforcement learning based training procedure is unstable due to its high variance of gradient estimation, which mainly comes from its two drawbacks. First, word predictions in the sentence are treated equally and receive the same reward, which ignores the fact that a bad translation is often due to translation errors in only a few positions, and predictions in other positions should not be held responsible. Second, as the reward is usually defined to be positive, the algorithm will always tend to raise the probability of the sampled sentence, leading to many inefficient parameter updates. There are solutions like reward shaping \cite{Ng1999PolicyIU} or baseline reward \cite{10.5555/647235.720252} to reduce the estimation variance. However, as the sampling cost of autoregressive models is very expensive, they will either lead to biased estimation \cite{ranzato2015sequence,bahdanau2016actor} or be time-consuming \cite{shen2016minimum,yu2017seqgan}.

\section{Sequence-Level Training for NAT}
In this article, we will present our sequence-level training methods in detail. We first discuss training methods based on the reinforcement learning, which can train NAT with any sequence-level objective that correlates well with the translation quality. We then introduce the differentiable bag-of-ngrams training objective to train NAT, which can be efficiently calculated without doing any approximations. Finally, we use a three-stage training strategy to combine the strengths of the two training methods and the word-level training.

\subsection{Reinforcement Learning}
The word-level cross-entropy loss cannot evaluate the output of NAT properly, so it is necessary to train NAT with sequence-level objectives (e.g., BLEU), which is more suitable for NAT models and correlates well with the translation quality. However, these objectives are usually discrete and therefore non-differentiable, so we propose to optimize the training objective under the reinforcement learning framework \cite{williams1992simple}. The basic idea is to define the loss function as the negative expected reward, and find an unbiased estimation of its gradient:
\begin{equation}
\begin{aligned}
\nabla_{\theta}\mathrm{L}_\theta=- \sum_{\bm{\mathrm{Y}}=y_{1:T}}\nabla_{\theta}{p(\bm{\mathrm{Y}}|\bm{\mathrm{X}},\theta) \cdot r(\bm{\mathrm{Y}})}\\
=-\sum_{\bm{\mathrm{Y}}}\nabla_{\theta}\prod_{t=1}^{T}p_t(y_t|\bm{X},\theta) \cdot r(\bm{\mathrm{Y}}) \label{eq:lossgrad}
\end{aligned}
\end{equation}

In the following, we first present the basic method \textbf{Reinforce-Base} in section \ref{rein:base}, which directly applies REINFORCE \cite{williams1992simple} to estimate Equation (\ref{eq:lossgrad}). Then we develop its improved version \textbf{Reinforce-Step} in section \ref{rein:step}, where the prediction in each time step has a unqiue step reward instead of sharing a sentence reward. \textbf{Reinforce-Topk} is further proposed in section \ref{rein:topk} to reduce the estimation variance by paying more attention to the top-ranking words, which are more important than others in the gradient estimation. Taking advantage of the fact that many reward functions are based on the comparison with the reference sentence, we finally propose \textbf{Traverse-Ref} in section \ref{rein:tr} to calculate the gradient accurately by only traversing words in the reference sentence. The time complexity of these methods is analyzed in section \ref{rein:time}.

\subsubsection{Reinforce-Base}
\label{rein:base}
%\vspace{5pt}
%\noindent{}\textbf{Reinforce-Base} 
We can put aside the characteristics of the NAT and follow autoregressive models to directly apply the REINFORCE algorithm to estimate the gradient:

\begin{equation}
\begin{aligned}
\label{eq:nonauto_base}
\nabla_{\theta}\mathrm{L}_\theta=- \mathop{\mathbb{E}}\limits_{\bm{\mathrm{Y}}} [\sum_{t=1}^{T}\nabla_{\theta} \log(p_t(y_t|\bm{\mathrm{X}},\theta)) \cdot r(\bm{\mathrm{Y}})]
\end{aligned}
\end{equation}

According to Equation (\ref{eq:nonauto_base}), the NAT model first generates the probability distribution and samples a sentence from the distribution. Then the reward of the sampled sentence is calculated to evaluate the translation quality. The loss function becomes the log-probability of the sampled sentence weighted by the reward, where the reward acts as the learning rate to enhance the learning of high-quality samples. Algorithm \ref{alg:base} describes the estimation process.

\begin{algorithm}[ht]
\caption{Reinforce-Base} 
\label{alg:base}
\hspace*{0.02in} {\bf Input:} 
probability distribution $p(\cdot|\bm{\mathrm{X}},\theta))$, length $T$\\
\hspace*{0.02in} {\bf Output:} 
the estimate of $\nabla_{\theta} \mathrm{L}_\theta$
\begin{algorithmic}[1]
\State  $\nabla_{\theta} \mathrm{L}_\theta=0$
\For{$t = 1$ {\bf{to}} $T$}
  \State sample $y_t$ from $p_t(\cdot|\bm{\mathrm{X}},\theta))$
  \State $\nabla_{\theta} \mathrm{L}_\theta$ += $-\nabla_{\theta}\log(p_t(y_t|\bm{\mathrm{X}},\theta))$
\EndFor
\State construct $\bm{Y}=\{y_1, ..., y{_T}\}$, calculate $r(\bm{\mathrm{Y}})$
\State $\nabla_{\theta} \mathrm{L}_\theta$ = $r(\bm{\mathrm{Y}}) \cdot \nabla_{\theta} \mathrm{L}_\theta$
\State \Return $\nabla_{\theta} \mathrm{L}_\theta$
\end{algorithmic}
\end{algorithm}

\subsubsection{Reinforce-Step}
\label{rein:step}
%\vspace{5pt}
%\noindent{}\textbf{Reinforce-Step} 
In Reinforce-Base, the word prediction $y_t$ in every time step $t$ receives the same sentence reward $r(\bm{\mathrm{Y}})$, ignoring the characteristic of independent generation in NAT models. Intuitively, since the word prediction in every step $t$ is independent of other steps, its reward should also only be related to the word $y_t$. Therefore, we derive the Equation (\ref{eq:lossgrad}) to the following form, where the gradient in each time step $t$ is weighted by the step reward $r(y_t)$, which is defined as the expectation of reward when the prediction in step $t$ is $y_t$:
\begin{theorem}
\begin{equation}
\begin{aligned}
\label{eq:nonauto_reduce}
\nabla_{\theta}\mathrm{L}_\theta=-\sum_{t=1}^{T}\sum_{y_t} \nabla_{\theta} p_t(y_t|\bm{\mathrm{X}},\theta) \cdot r(y_t)\\
\end{aligned}
\end{equation}
\begin{equation}
\begin{aligned}
\label{eq:reward}
\text{where}\ \ \ r(y_t) = \mathop{\mathbb{E}}\limits_{y_{1:t-1}}\mathop{\mathbb{E}}\limits_{y_{t+1:T}}r(\bm{\mathrm{Y}})
\end{aligned}
\end{equation}
\end{theorem}
\begin{proof}
\begin{equation}
\begin{aligned}
\nabla_{\theta}\mathrm{L}_\theta=-&\sum_{\bm{\mathrm{Y}}}\nabla_{\theta}\prod_{t=1}^{T}p_t(y_t|\bm{X},\theta) \cdot r(\bm{\mathrm{Y}})\\
=-&\sum_{\bm{\mathrm{Y}}}\sum_{t=1}^{T}\nabla_{\theta}p_t(y_t|\bm{X},\theta)\cdot\prod_{i=1}^{t-1}p_i(y_i|\bm{X},\theta)\cdot\prod_{j=t+1}^{T}p_j(y_j|\bm{X},\theta)\cdot r(\bm{\mathrm{Y}})\\
=-&\sum_{t=1}^{T}\sum_{\bm{\mathrm{Y}}}\nabla_{\theta}p_t(y_t|\bm{X},\theta)\cdot\prod_{i=1}^{t-1}p_i(y_i|\bm{X},\theta)\cdot\prod_{j=t+1}^{T}p_j(y_j|\bm{X},\theta)\cdot r(\bm{\mathrm{Y}})\\
=-&\sum_{t=1}^{T}\sum_{{y_t}}\nabla_{\theta}p_t(y_t|\bm{X},\theta)\cdot\sum_{{y_{1:t-1}}}\sum_{y_{t+1:T}}\prod_{i=1}^{t-1}p_i(y_i|\bm{X},\theta)\cdot\prod_{j=t+1}^{T}p_j(y_j|\bm{X},\theta)\cdot r(\bm{\mathrm{Y}})\\
=-&\sum_{t=1}^{T}\sum_{{y_t}}\nabla_{\theta}p_t(y_t|\bm{X},\theta)\cdot\mathop{\mathbb{E}}\limits_{y_{1:t-1}}\mathop{\mathbb{E}}\limits_{y_{t+1:T}}r(\bm{\mathrm{Y}}).\\
=-&\sum_{t=1}^{T}\sum_{y_t} \nabla_{\theta} p_t(y_t|\bm{\mathrm{X}},\theta) \cdot r(y_t)
\end{aligned}
\end{equation}
\end{proof}

In Equation (\ref{eq:nonauto_reduce}), each time step receives an accurate step reward $r(y_t)$ instead of sharing a sentence reward $r(\bm{\mathrm{Y}})$. We can still apply the REINFORCE algorithm to estimate the gradient:

\begin{equation}
\begin{aligned}
\nabla_{\theta}\mathrm{L}_\theta=- \sum_{t=1}^{T}\mathop{\mathbb{E}}\limits_{y_t} [\nabla_{\theta} \log(p_t(y_t|\bm{\mathrm{X}},\theta)) \cdot r(y_t)]
\end{aligned}
\end{equation}

The estimation process of Reinforce-Step is basically the same as Reinforce-Base, except that each word prediction receives a step reward rather than the sentence reward. The step reward is defined as the expectation of reward when the prediction in step $t$ is specified, which can be estimated by Monte Carlo sampling, as illustrated in algorithm \ref{alg:reward}. Specifically, we fix the prediction $y_t$ in step $t$ and sample words of other time steps from the probability distribution $p(\cdot|\bm{\mathrm{X}},\theta))$. After obtaining the sentence, we calculate the sentence reward $r(\bm{\mathrm{Y}})$. We repeat this process for $n$ times and use the average reward to estimate the step reward $r(y_t)$. Algorithm \ref{alg:step} describes the process of Reinforce-Step. 

The idea of assigning a unique reward to each time step has similarities with the actor-critic approach in NMT \cite{ranzato2015sequence,bahdanau2016actor}, which uses a critic network to predict the expected reward $r(y_{1:t})$ after generating $t$ words. In comparison, the step reward $r(y_{t})$ in Reinforce-Step is more accurate since it does not depend on the previously generated words. Besides, due to the bias of the neural network prediction, the gradient estimation of the actor-critic approach is biased. In comparison, Reinforce-Step can obtain an unbiased estimation of the gradient.

\begin{algorithm}[ht]
\caption{Estimation of step reward $r(y_t)$} 
\label{alg:reward}
\hspace*{0.02in} {\bf Input:}
probability distribution $p(\cdot|\bm{\mathrm{X}},\theta))$, step $t$, word $y_t$, length $T$, sampling times $n$\\
\hspace*{0.02in} {\bf Output:} 
the estimate of $r(y_t)$
\begin{algorithmic}[1]
\State $r$ = $0$
\For{$i = 1$ {\bf{to}} $n$}
  \State sample ${{y}}_{1:t-1}$, ${{y}}_{t+1:T}$ from $p(\cdot|\bm{\mathrm{X}},\theta))$
  \State construct ${\bm{\mathrm{Y}}}$ = $\{{{y}}_{1:t-1}$, $y_t$, ${{y}}_{t+1:T}\}$, calculate $r(\bm{\mathrm{Y}})$
  \State $r$ += $r({\bm{\mathrm{Y}}})$
\EndFor
\State $r$ = $r/n$
\State \Return $r$
\end{algorithmic}
\end{algorithm}
\begin{algorithm}[ht]
\caption{Reinforce-Step} 
\label{alg:step}
\hspace*{0.02in} {\bf Input:} 
probability distribution $p(\cdot|\bm{\mathrm{X}},\theta))$, length $T$, sampling times $n$\\
\hspace*{0.02in} {\bf Output:} 
the estimate of $\nabla_{\theta} \mathrm{L}_\theta$
\begin{algorithmic}[1]
\State  $\nabla_{\theta} \mathrm{L}_\theta=0$
\For{$t = 1$ {\bf{to}} $T$}
  \State sample $y_t$ from $p_t(\cdot|\bm{\mathrm{X}},\theta))$
  \State estimate $r(y_t)$ by algorithm \ref{alg:reward} with sampling times $n$
  \State $\nabla_{\theta} \mathrm{L}_\theta$ += $- r(y_t) \cdot \nabla_{\theta} \log(p_t(y_t|\bm{\mathrm{X}},\theta))$
\EndFor
\State \Return $\nabla_{\theta} \mathrm{L}_\theta$
\end{algorithmic}
\end{algorithm}

\subsubsection{Reinforce-Topk}
\label{rein:topk}
%\vspace{5pt}
%\noindent{}\textbf{Reinforce-Topk} 
The estimation variance of Reinforce-Step can be further reduced if we can traverse the vocabulary to directly calculate the Equation (\ref{eq:nonauto_reduce}) instead of applying REINFORCE for estimation. However, this will bring a high computational cost due to the large vocabulary size. Therefore, we take a step back and only traverse a subset of the vocabulary. The subset contains important words for the gradient estimation and filters out unimportant words, so we can effectively reduce the estimation variance and meantime maintain the acceptable training speed.

The probability distribution over the target vocabulary is usually a centered distribution where the top-ranking words occupy the central part of the distribution, and the softmax layer ensures that the other words with small probabilities have small gradients. Hence the variance will be effectively reduced if we can eliminate the variance of top-ranking words. This motivates us to compute gradients of the top-ranking words accurately and estimate the rest via the REINFORCE algorithm.

Firstly, we denote the subset of target words with top-k probabilities in step $t$ as $T_k^t$. As defined in Equation (\ref{eq:def}), $P_k^t$ is the sum of probabilities in $T_k^t$, and $\tilde p$ is the normalized probability distribution after removing the words in $T_k^t$:
\begin{equation}
\label{eq:def}
P_k^t = \sum_{y_t \in T_k^t}p_t(y_t|\bm{\mathrm{X}},\theta),\ \  \tilde p_t(y_t|\bm{\mathrm{X}},\theta)= \left\{
             \begin{array}{lr}
             0, & y \in T_k^t \\
            \frac{p_t(y_t|\bm{\mathrm{X}},\theta)}{1-P_k^t}, & y \notin T_k^t
             \end{array}
\right.
\end{equation}
We can divide the gradient into two parts, and process them in different ways. For the important words in $T_k^t$, we traverse them and accumulate their gradients. For other words, we estimate their gradients by one sampling from the normalized distribution $\tilde p$, and weight the estimation by $1-P_k^t$. The following equation shows that an unbiased estimate of the gradient can be obtained in this way, which effectively reduce the estimation variance with the acceptable training cost. Algorithm \ref{alg:topk} describes the process of Reinforce-Topk.

\begin{theorem}
\begin{equation}
%\begin{aligned}
\label{eq:rfnat}
\nabla_{\theta}\mathrm{L}_\theta=-\sum_{t=1}^{T}(\sum_{y_t \in T_k^t}\nabla_{\theta} p_t(y_t|\bm{\mathrm{X}},\theta)\cdot r(y_t)+(1-P_k^t) \cdot \mathop{\mathbb{E}}_{y_t\sim \tilde p}[\nabla_{\theta} \log(p_t(y_t|\bm{\mathrm{X}},\theta)) \cdot r(y_t)])
%\end{aligned}
\end{equation}
\end{theorem}

\begin{proof}
\begin{equation}
\begin{aligned}
\nabla_{\theta}\mathrm{L}_\theta=-&\sum_{t=1}^{T}\sum_{y_t} \nabla_{\theta} p_t(y_t|\bm{\mathrm{X}},\theta) \cdot r(y_t)\\
=-&\sum_{t=1}^{T}(\sum_{y_t \in T_k^t} \nabla_{\theta} p_t(y_t|\bm{\mathrm{X}},\theta) \cdot r(y_t) + \sum_{y_t \notin T_k^t} \nabla_{\theta} p_t(y_t|\bm{\mathrm{X}},\theta) \cdot r(y_t))\\
=-&\sum_{t=1}^{T}(\sum_{y_t \in T_k^t} \nabla_{\theta} p_t(y_t|\bm{\mathrm{X}},\theta) \cdot r(y_t) + (1-P_k^t) \cdot \sum_{y_t \notin T_k^t}\frac{p_t(y_t|\bm{\mathrm{X}},\theta)}{1-P_k^t}\nabla_{\theta} \log(p_t(y_t|\bm{\mathrm{X}},\theta)) \cdot r(y_t))\\
=-&\sum_{t=1}^{T}(\sum_{y_t \in T_k^t} \nabla_{\theta} p_t(y_t|\bm{\mathrm{X}},\theta) \cdot r(y_t) + (1-P_k^t) \cdot \mathop{\mathbb{E}}_{y_t\sim \tilde p}[\nabla_{\theta} \log(p_t(y_t|\bm{\mathrm{X}},\theta)) \cdot r(y_t)])
\end{aligned}
\end{equation}
\end{proof}

\begin{algorithm}[ht]
\caption{Reinforce-Topk} 
\label{alg:topk}
\hspace*{0.02in} {\bf Input:} 
probability distribution $p(\cdot|\bm{\mathrm{X}},\theta))$, topk size $k$, length $T$, sampling times $n$ \\
\hspace*{0.02in} {\bf Output:} 
the estimate of $\nabla_{\theta} \mathrm{L}_\theta$
\begin{algorithmic}[1]
\State  $\nabla_{\theta} \mathrm{L}_\theta=0$
\For{$t = 1$ {\bf{to}} $T$}
  \State $T_k^t$ = \{words ranking top-$k$ in $p_t(\cdot|\bm{\mathrm{X}},\theta))$\}
  \State $\tilde p_t = p_t$, $P_k^t=0$
  \For{$y_t$ in $T_k^t$}
    \State estimate $r(y_t)$ by algorithm \ref{alg:reward} with sampling times $n$
    \State $\nabla_{\theta} \mathrm{L}_\theta$ += $- \nabla_{\theta} p_t(y_{t}|\bm{\mathrm{X}},\theta)\cdot r(y_t)$
    \State $\tilde p_t(y_{t}|\bm{\mathrm{X}},\theta)$ = 0, $P_k$ += $p_t(y_{t}|\bm{\mathrm{X}},\theta)$
  \EndFor
  \State $\tilde p_t$ = $\tilde p_t / (1-P_k)$, sample $y_t$ from $\tilde p_t(\cdot|\bm{\mathrm{X}},\theta)$
  \State estimate $r(y_t)$ by algorithm \ref{alg:reward} with sampling times $n$
  \State $\nabla_{\theta} \mathrm{L}_\theta$ += $- (1-P_k) \cdot \nabla_{\theta} \log(p_t(y_{t}|\bm{\mathrm{X}},\theta))\cdot r(y_t)$
\EndFor
\State \Return $\nabla_{\theta} \mathrm{L}_\theta$
\end{algorithmic}
\end{algorithm}

\subsubsection{Traverse-Ref}
\label{rein:tr}
%\vspace{5pt}
%\noindent{}\textbf{Traverse-Ref}
Due to the large vocabulary size, it is costly to directly calculate the gradient according to Equation (\ref{eq:nonauto_reduce}), so we have to rely on reinforcement learning algorithms for gradient estimation. However, we find that when the reward function is \textit{Reference-Based}, it becomes possible to traverse words in the whole vocabulary and calculate their rewards in a short time, enabling us to directly calculate the gradient according to Equation (\ref{eq:nonauto_reduce}). Firstly, we call a word out-of-$\bm{\mathrm{X}}$ if it does not appear in sentence $\bm{\mathrm{X}}$, and define the Reference-Based reward as follows:
\begin{definition}
A reward function $r(\bm{\mathrm{Y}})$ is Reference-Based if it evaluates $\bm{\mathrm{Y}}$ by comparing it with a reference sentence $\hat{\bm{\mathrm{Y}}}$ and the reward do not change when we replace any out-of-$\hat{\bm{\mathrm{Y}}}$ word in $\bm{\mathrm{Y}}$ by other out-of-$\hat{\bm{\mathrm{Y}}}$ words.
\end{definition}

From the definition, we can see that many widely used reward functions are Reference-Based. For example, since words that do not appear in the reference sentence can never match the reference, the rewards based on n-gram matching (e.g., BLEU, GLEU, and ROUGE) will not be changed by replacing any out-of-reference word by other out-of-reference words, so these rewards are Reference-Based. Recall that the step reward $r(y_t)$ is defined by the expectation of reward when the prediction in step $t$ is $y_t$. According to the definition of Reference-Based, the step reward $r(y_t)$ will take the same value as long as the reward function is Reference-Based and $y_t$ does not appear in the reference sentence. Therefore, we can divide the vocabulary into two parts. For words in the reference sentence, we traverse these words and estimate their step rewards. For out-of-reference words, we only need to estimate the step reward once since they take the same value. Finally, we calculate the gradient according to Equation (\ref{eq:nonauto_reduce}). In this way, the variance caused by the reinforcement learning method is completely eliminated, so the estimation variance only comes from the estimation of the step reward. Algorithm \ref{alg:trav} describes the process of Traverse-Ref. 

Notice that Reinforce-Topk and Traverse-Ref are not applicable to sentence-level reward. Reinforce-Topk works for word-level reward since the top-k words can usually occupy a large part of probability distribution. However, considering the exponential search space, traversing top-k sentences generally does not have a great impact on gradient estimation. Traverse-Ref works for word-level reward by reducing the search space of vocabulary size $V$ to $T$ reference words and one out-of-reference word. However, for sentence reward, the search space is only reduced from $V^T$ to $(T+1)^{T}$, which is still intractable.
\begin{algorithm}[ht]
\caption{Traverse-Ref} 
\label{alg:trav}
\hspace*{0.02in} {\bf Input:} 
probability distribution $p(\cdot|\bm{\mathrm{X}},\theta))$, length $T$, sampling times $n$, target vocabulary $V$, reference sentence $\hat{\bm{\mathrm{Y}}}$\\
\hspace*{0.02in} {\bf Output:} 
the estimate of $\nabla_{\theta} \mathrm{L}_\theta$\
\begin{algorithmic}[1]
\State  $\nabla_{\theta} \mathrm{L}_\theta=0$
\State $V_{\hat{\bm{\mathrm{Y}}}}$ = \{words in the reference sentence $\hat{\bm{\mathrm{Y}}}$\}
\State randomly choose a word $w$ from $V\setminus V_{\hat{\bm{\mathrm{Y}}}}$
\For{$t = 1$ {\bf{to}} $T$}
  \For {$y_t$ in $V_{\hat{\bm{\mathrm{Y}}}}$}
    \State estimate $r(y_t)$ by algorithm \ref{alg:reward} with sampling times $n$
  \EndFor
  \State estimate $r(w)$ by algorithm \ref{alg:reward} with sampling times $n$
  \For {$y_t$ in $V\setminus V_{\hat{\bm{\mathrm{Y}}}}$}
    \State $r(y_t)$ = $r(w)$
  \EndFor
  \For {$y_t$ in $V$}
  \State $\nabla_{\theta}\mathrm{L}_\theta$ += $-\nabla_{\theta} p_t(y_t|\bm{\mathrm{X}},\theta) \cdot r(y_t)$
  \EndFor
\EndFor
\State \Return $\nabla_{\theta} \mathrm{L}_\theta$
\end{algorithmic}
\end{algorithm}

\subsubsection{Time Complexity}
\label{rein:time}
The proposed methods will not affect the time complexity on GPU, but the computational cost on CPU becomes non-negligible since we have to calculate the reward for many times. We take the calculation of the reward as a time unit and give the time complexity of the proposed methods in the Table \ref{tab:comp}. Generally, the time complexity increases as the algorithm evolves from Reinforce-Base to Traverse-Ref.
\begin{table}[ht]
\centering
\caption{Time complexity of the proposed methods on CPU. $T$ is the length of the target sentence, $n$ is the sampling times when estimating the step reward, and $k$ is the topk size in Reinforce-Topk.}
\begin{tabular}{c|c|c|c|c}
\toprule
&Reinforce-Base&Reinforce-Step&Reinforce-Topk&Traverse-Ref\\
\midrule
Time Complexity&$\mathcal{O}(1)$&$\mathcal{O}(nT)$&$\mathcal{O}(nkT)$&$\mathcal{O}(nT^2)$\\
\bottomrule
\end{tabular}
\label{tab:comp}
\end{table}

\subsection{Bag-of-Ngrams}
Reinforcement learning based methods optimize the sequence-level objective through the gradient estimation. To stabilize the training process, we make many efforts to reduce the estimation variance. However, this requires a lot of reward calculation and hence the computational cost on CPU becomes large. In this section, we introduce a novel training objective based on bag-of-ngrams for NAT, which is differentiable and can be efficiently calculated without doing any approximations.

Bag-of-Words (BoW) \cite{joachims1998text} is a widely used text representation model that discards the word order and represents a sentence as the multiset of its belonging words. Bag-of-ngrams \cite{pang-etal-2002-thumbs,li-etal-2016-weighted} is proposed to enhance the text representation by taking consecutive words (n-gram) into consideration. Besides, bag-of-ngrams also plays an important role in the evaluation of translation quality. Recall those evaluation metrics that can evaluate the translation quality well (e.g., BLEU, GLEU, and ROUGE), many of them are based on the accuracy or recall of n-grams, which basically depends on the intersection size of bag-of-ngrams. Therefore, we propose to directly train NAT to minimize the Bag-of-Ngrams (BoN) difference between the NAT output and reference. We first define the BoN of a discrete sentence by the sum of n-gram vectors with one-hot representation. Then we define the BoN of NMT by the expectation of BoN on all possible translations and give an efficient method to calculate the BoN of NAT. Finally, we give methods to calculate the BoN distance between the NAT output and reference.

\subsubsection{Definition}
%\vspace{5pt}
%\noindent{}\textbf{Definition} 
Bag-of-ngrams is the sum of vectors where each vector is the one-hot representation of an n-gram, which has the size $|V|^n$ when the vocabulary size is $|V|$. Formally, for a sentence $\bm{Y}=\{y_1, ..., y{_T}\}$, we use $\text{BoN}_{\bm{Y}}$ to denote the bag-of-ngrams of $\bm{Y}$. For an n-gram $\bm{g} = (g_1,\dots,g_n)$, we use $\text{BoN}_{\bm{Y}}(\bm{g})$ to denote the value of entry $\bm{g}$ in $\text{BoN}_{\bm{Y}}$, which is the number of occurrences of n-gram $\bm{g}$ in sentence $\bm{Y}$ and is formulized as follows:
\begin{equation}
\label{eq:bon}
\text{BoN}_{\bm{Y}}(\bm{g}) =\sum_{t=0}^{T-n}1\{y_{t+1:t+n} = \bm{g}\}
\end{equation}
where $1\{\cdot\}$ is the indicator function that takes value from $\{0,1\}$ whose value is 1 if and only if the inside condition holds. 

For a discrete sentence, our definition of BoN is consistent with previous work. However, there is no clear definition of BoN for sequence models like NMT, which model the probability distribution on the whole target space. A natural approach is to consider all possible translations and use the expected BoN to define the BoN for sequence models. For NMT with parameter $\theta$, we use $\text{BoN}_{\theta}$ to denote its bag-of-ngrams. Formally, given a source sentence $\bm{X}$, the value of entry $\bm{g}$ in $\text{BoN}_{\theta}$ is defined as follows:
\begin{equation}
\label{eq:bon_theta}
\text{BoN}_{\theta}(\bm{g}) =\sum_{\bm{Y}} P(\bm{Y}|\bm{X},\theta)\cdot \text{BoN}_{\bm{Y}}(\bm{g})
\end{equation}
\subsubsection{Efficient Calculation}
%\vspace{5pt}
%\noindent{}\textbf{Efficient Calculation}
It is unrealistic to directly calculate $\text{BoN}_{\bm{Y}}(\bm{g})$ according to Equation (\ref{eq:bon_theta}) due to the exponential search space. For autoregressive NMT, because of the conditional dependency in modeling translation probability, it is difficult to simplify the calculation without loss of accuracy. Fortunately, NAT models the translation probability in different positions independently, which enables us to divide the target sequence into subareas and analyze the BoN in each subarea without being influenced by other positions. Using this unique property of NAT, we can convert Equation (\ref{eq:bon_theta}) to the following form:
\begin{theorem}
\begin{equation}
\label{eq:bon_theta_reduce}
\text{BoN}_{\theta}(\bm{g})=\sum_{t=0}^{T-n}\prod_{i=1}^{n}p_{t+i}(y_{t+i}=g_i|\bm{X},\theta)
\end{equation}
\end{theorem}
\begin{proof}
\begin{equation}
\begin{aligned}
\text{BoN}_{\theta}&(\bm{g})=\sum_{\bm{Y}}P(\bm{Y}|\bm{X},\theta)\cdot\sum_{t=0}^{T-n}1\{y_{t+1:t+n} = \bm{g}\}\\
=&\sum_{t=0}^{T-n}\sum_{\bm{Y}} P(\bm{Y}|\bm{X},\theta) \cdot 1\{y_{t+1:t+n} = \bm{g}\}\\
=&\sum_{t=0}^{T-n}\sum_{\bm{Y}_{1:t}}P(\bm{Y}_{1:t}|\bm{X},\theta)\sum_{\bm{Y}_{t+n+1:T}}P(\bm{Y}_{t+n+1:T}|\bm{X},\theta)\sum_{\bm{Y}_{t+1:t+n}} P(\bm{Y}_{t+1:t+n}|\bm{X},\theta) \cdot 1\{y_{t+1:t+n} = \bm{g}\}\\
=&\sum_{t=0}^{T-n}\sum_{\bm{Y}_{t+1:t+n}} P(\bm{Y}_{t+1:t+n}|\bm{X},\theta) \cdot 1\{y_{t+1:t+n} = \bm{g}\}\\
=&\sum_{t=0}^{T-n}\prod_{i=1}^{n}p_{t+i}(y_{t+i}=g_i|\bm{X},\theta)
\end{aligned}
\end{equation}
\end{proof}
Equation (\ref{eq:bon_theta_reduce}) gives an efficient method to calculate $\text{BoN}_{\theta}(\bm{g})$, where we slide a window on NAT output distributions to obtain all continue subareas of size $n$, and then accumulate the counts of n-gram $\bm{g}$ in all subareas. This process does not make any approximation and requires little computational effort. 

\subsubsection{Bag-of-Words Objective}
%\vspace{5pt}
%\noindent{}\textbf{Bag-of-Words Objective} 
Our objective is to minimize the BoN difference between NAT output and reference. The difference can be measured by many metrics such as the $L_1$ distance, $L_2$ distance and cosine distance. BoN is defined to be a vector of size $|V|^n$ where $|V|$ is the vocabulary size. Though we have an efficient calculation method for $\text{BoN}_{\theta}(\bm{g})$, computing the complete BoN vector for NAT is still unaffordable due to the large BoN size. The only exception is the case of $n=1$, where the bag-of-ngrams degenerates into Bag-of-Words (BoW). In this situation, we only need to sum the probability distributions in all time steps to obtain $\text{BoW}_{\theta}$ and apply distance metrics to calculate BoW distances like BoW-$L_1$, BoW-$L_2$ and BoW-$Cos$.

\subsubsection{Bag-of-Ngrams Objective}
%\vspace{5pt}
%\noindent{}\textbf{Bag-of-Ngrams Objective} 
For $n>1$, the complete BoN vector is unavailable, so many distance metrics like $L_2$ distance and cosine distance cannot be calculated.
Fortunately, we find that the $L_1$ distance between the two BoN vectors, denoted as BoN-$L_1$, can be simplified using the sparsity of bag-of-ngrams. As shown in Equation (\ref{eq:bon_theta_reduce}), for NAT, its bag-of-ngrams vector $\text{BoN}_{\theta}$ is dense. On the contrary, assume that the reference sentence is $\hat{\bm{Y}}$, the vector $\text{BoN}_{\hat{\bm{Y}}}$ is very sparse where only a few entries of it have non-zero values. Using this property, we can write BoN-$L_1$ as follows:

\begin{theorem}
\begin{equation}
\label{eq:bon_l1}
\text{BoN-}L_1=2(T-n+1-\sum_{\bm{g}}\min(\text{BoN}_{\theta}(\bm{g}),\text{BoN}_{\hat{\bm{Y}}}(\bm{g})))
\end{equation}
\end{theorem}
\begin{proof}
First, we show that the $L_1$-norm of $\text{BoN}_{\bm{Y}}$ and $\text{BoN}_{\theta}$ are both $T-n+1$:
\begin{equation}
\begin{aligned}
\label{eq:l1-norm}
\sum_{\bm{g}}\text{BoN}&_{\bm{Y}}(\bm{g})=\sum_{t=0}^{T-n}\sum_{\bm{g}}1\{y_{t+1:t+n} = \bm{g}\}=T-n+1\\
\sum_{\bm{g}}\text{BoN}&_{\theta}(\bm{g})=\sum_{\bm{g}}\sum_{\bm{Y}} P(\bm{Y}|\bm{X},\theta)\cdot \text{BoN}_{\bm{Y}}(\bm{g})\\
&=\sum_{\bm{Y}} P(\bm{Y}|\bm{X},\theta)\cdot \sum_{\bm{g}}\text{BoN}_{\bm{Y}}(\bm{g})=T-n+1
\end{aligned}
\end{equation}
On this basis, we can convert BoN-$L_1$ to the following form:
\begin{equation}
\begin{aligned}
\text{BoN-}&L_1=\sum_{\bm{g}}\vert\text{BoN}_{\theta}(\bm{g})-\text{BoN}_{\hat{\bm{Y}}}(\bm{g})\vert\\
&=\sum_{\bm{g}}(\text{BoN}_{\theta}(\bm{g})+\text{BoN}_{\hat{\bm{Y}}}(\bm{g})-2\min(\text{BoN}_{\theta}(\bm{g}),\text{BoN}_{\hat{\bm{Y}}}(\bm{g}))\\
&=2(T-n+1-\sum_{\bm{g}}\min(\text{BoN}_{\theta}(\bm{g}),\text{BoN}_{\hat{\bm{Y}}}(\bm{g})))
\end{aligned}
\end{equation}
\end{proof}
The minimum between $\text{BoN}_{\theta}(\bm{g})$ and $\text{BoN}_{\hat{\bm{Y}}}(\bm{g})$ can be understood as the number of matches for the n-gram $\bm{g}$, and the $L_1$ distance measures the number of n-grams predicted by NAT that fails to match the reference sentence. Notice that the minimum will be nonzero only if the n-gram $\bm{g}$ appears in the reference sentence. Hence we can only focus on n-grams in the reference, which significantly reduces the computational effort and storage requirement. Algorithm \ref{alg:bon} illustrates the calculation process of $\text{BoN-}L_1$.
\begin{algorithm}[ht]
\caption{BoN-$L_1$} 
\label{alg:bon}
\hspace*{0.02in} {\bf Input:} 
probability distribution $p(\cdot|\bm{\mathrm{X}},\theta))$, reference sentence $\hat{\bm{Y}}$, prediction length $T$, $n$ \\
\hspace*{0.02in} {\bf Output:} 
BoN distance BoN-$L_1$
\begin{algorithmic}[1]
\State construct the bag-of-ngrams $\text{BoN}_{\hat{\bm{Y}}}$ for the reference sentence
\State ref-ngrams=\{$\bm{g}|\text{BoN}_{\hat{\bm{Y}}}(\bm{g})$ != $0$\}
\State match = $0$
\For{$\bm{g}$ in ref-ngrams}
  \State calculate $\text{BoN}_{\theta}(\bm{g})$ according to Equation (\ref{eq:bon_theta_reduce})
  \State match +=  $\min(\text{BoN}_{\theta}(\bm{g}),\text{BoN}_{\hat{\bm{Y}}}(\bm{g}))$
\EndFor
\State BoN-$L_1$ = $2$ $\cdot$ ($T$-$n$+$1$-match)
\State \Return BoN-$L_1$
\end{algorithmic}
\end{algorithm}

We normalize distances to range $[0,1]$ as training objectives. For BoW distances, we keep BoW-$Cos$ unchanged and divide BoW-$L_1$ and BoW-$L_2$ by the constant $2T$:
\begin{equation}
\label{eq:bownorm}
\mathcal{L}_{\text{BoW-}{L_1}}(\theta) = \frac{\text{BoW-}L_1}{2T}\quad \quad \mathcal{L}_{\text{BoW-}{L_2}}(\theta) = \frac{\text{BoW-}L_2}{2T}
\end{equation}
For BoN, we divide BoN-$L_1$ by the constant $2(T-n+1)$:
\begin{equation}
\label{eq:bonnorm}
\mathcal{L}_{\text{BoN-}{L_1}}(\theta) = \frac{\text{BoN-}L_1}{2(T-n+1)}
\end{equation}

\subsection{Training Strategy}
The reinforcement learning method can train NAT with any sequence-level objective that correlates well with the translation quality, but it requires a lot of calculations on CPU to reduce the variance of gradient estimation. The bag-of-ngrams method can efficiently calculate the BoN objective without doing any approximations, but the training objective is limited to the $L_1$ distance. The word-level cross-entropy loss cannot evaluate the output of NAT properly, but it also has strengths like high-speed training and it is suitable for model warmup.

Therefore, we propose to use a three-stage training strategy to combine the strengths of the two training methods and the cross-entropy loss. Firstly, we use the cross-entropy loss to pretrain the NAT model, and this process takes the most training steps. Then we use the bag-of-ngrams objective to finetune the pretrained model for a few training steps. Finally, we apply the reinforcement learning method to finetune the model to optimize the sequence-level objective, where this process takes the least training steps. There are also other training strategies like two-stage training and joint training, and we will show the efficiency of three-stage training in the experiment. The loss based on reinforcement learning or bag-of-ngrams can also be used alone to finetune the model pretrained by the cross-entropy loss. We will adopt this strategy when analyzing these methods separately.

\section{Related Work}
\subsection{Non-Autoregressive Translation}
\citet{gu2017non} proposed non-autoregressive translation to reduce the translation latency, which generates all target tokens simultaneously. Although accelerating the decoding process significantly, the acceleration comes at the cost of translation quality. Therefore, intensive efforts have been devoted to improving the performance of NAT, which can be roughly divided into the following categories.

\vspace{5pt}
\noindent{}\textbf{Latent Variables.} NAT suffers from the multimodality problem, which can be mitigated by introducing a latent variable to directly model the nondeterminism in the translation process. \citet{gu2017non} proposed to use fertility scores specifying the number of output words each input word generates to model the latent variable. \citet{kaiser2018fast} autoencoded the target sequence into a sequence of discrete latent variables and decoded the output sequence from the latent sequence in parallel. Based on variational inference, \citet{Ma_2019} proposed FlowSeq to model sequence-to-sequence generation using generative flow, and \citet{Shu2020LatentVariableNN} introduced LaNMT with continuous latent variables and deterministic inference. \citet{bao2019nonautoregressive,ran2019guiding} used the position information as latent variables to explicitly model the reordering information in the decoding procedure. 

\vspace{5pt}
\noindent{}\textbf{Decoding Methods.} The fully non-autoregressive transformer generates all target words in one run, which suffers from the large performance degradation. Therefore, researchers were interested in alternative decoding methods that are slightly slower but can significantly improve the translation quality. \citet{lee2018deterministic} proposed the iterative decoding based NAT model IRNAT to iteratively refine the translation where the outputs of the decoder are fed back as the decoder inputs in the next iteration. The pattern of iterative decoding was adopted by many non-autoregressive models. \citet{ghazvininejad2019maskpredict,Kasai2020NonautoregressiveMT} refine model output iteratively by masking part of the translation and predicting the masks in each iteration. \citet{gu2019levenshtein} introduced the Levenshtein Transformer to iteratively refine the translation with insertion and deletion operations. In addition to iterative decoding, \citet{NIPS2019_8566} proposed to incorporate the Conditional Random Fields in the top of NAT decoder to help the NAT decoding. \citet{wang2018semi} introduced the semi-autoregressive decoding mechanism that generates a group of words each time. \citet{ran-etal-2020-learning} proposed another semi-autoregressive model named RecoverSAT, which generates a translation as a sequence of simultaneously generated segments.

\vspace{5pt}
\noindent{}\textbf{Training Objectives.} As the cross-entropy loss can not evaluate NAT outputs properly, researchers attempt to improve the model performance by introducing better training objectives. In addition to sequence-level training \cite{shao-etal-2019-retrieving,DBLP:conf/aaai/ShaoZFMZ20}, \citet{wang2019non} proposed the similarity regularization and reconstruction regularization to reduce errors of repeated and incomplete translations. \citet{libovicky2018end,saharia-etal-2020-non} applied the Connectionist Temporal Classification loss to marginalizes out latent alignments using dynamic programming. \citet{Aligned} proposed the aligned cross-entropy loss, which uses a differentiable dynamic program based on the best monotonic alignment between target tokens and model predictions.

\vspace{5pt}
\noindent{}\textbf{Other Improvements.} Besides the above mentioned categories, some researchers improve the NAT performance from other perspectives. \citet{guo2019non} proposed to enhance the inputs of NAT decoder with phrase-table lookup and embedding mapping. \citet{akoury-etal-2019-syntactically} introduced syntactically supervised Transformers, which first autoregressively predicts a chunked parse tree and then generate all target tokens conditioned on it. \citet{zhou-keung-2020-improving} proposed to improve NAT performance with source-side monolingual data. \citet{shan2021modeling} proposed to model the coverage information for NAT. \citet{li-etal-2019-hint,wei-etal-2019-imitation} improve the performance of NAT by exploring better methods to learn from autoregressive models. \citet{Zhou2020Understanding} investigated the knowledge distillation technique in NAT. \citet{tu-etal-2020-engine} introduced the energy-based inference networks as an alternative to knowledge distillation.

\subsection{Sequence-Level Training for Autoregressive NMT}
Neural machine translation models are usually trained with the word-level loss under the teacher forcing algorithm \cite{williams1989learning}, which forces the model to generate the next word based on the previous ground-truth words other than the model outputs during the training. However, this training method suffers from the {\em exposure bias} problem \cite{ranzato2015sequence} since the model is exposed to different data distributions during training and inference. To alleviate the exposure bias problem, some researchers improve the teacher forcing algorithm to professor forcing \cite{professor} or seer forcing \cite{feng2021guiding}. Scheduled sampling \cite{bengio2015scheduled,venkatraman2015improving} is the direct solution for exposure bias, which attempts to alleviate the exposure bias problem through mixing ground-truth words and previously predicted words as inputs during the training. However, the generated sequence may not be aligned with the target sequence, which is inconsistent with the word-level loss. Therefore, it is a natural solution to apply sequence-level training to eliminate the exposure bias in the autoregressive NMT.

Sequence-level training objectives are usually non-differentiable, and reinforcement learning techniques \cite{williams1992simple,sutton2000policy} are widely applied to train autoregressive NMT to optimize discrete objectives. \citet{ranzato2015sequence} first pointed out the exposure bias problem and proposed the MIXER algorithm to alleviate the exposure bias, which is a combination of the word-level cross-entropy loss and the sequence-level loss optimized by the REINFORCE algorithm. \citet{bahdanau2016actor} presented an approach to training neural networks to generate sequences using actor-critic methods from reinforcement learning. 
\citet{he2016dual} proposed a dual learning approach to train the forward NMT model using reward signals provided by the backward model. \citet{wu2016google} introduced a new sequence evaluation metrics GLEU for the sequence-level training of Google's NMT System. \citet{yu2017seqgan} proposed a sequence generation framework called SeqGAN to overcome the differentiable difficulty of GAN through reinforcement learning, which is then applied by \citet{pmlr-v95-wu18a,yang-etal-2018-improving} to train NMT under the generator-discriminator framework. \citet{wu-etal-2018-study} conducted a systematic study on the reinforcement learning based training method for NMT. 

In addition to reinforcement learning based methods, there are also some approaches that can train NMT with sequence-level objectives. \citet{shen2016minimum} 
introduced Minimum Risk Training for NMT to minimize the expected risk on training data. \citet{10.5555/3157096.3157290} proposed Reward Augmented Maximum Likelihood to incorporate sequence-level reward into a maximum likelihood framework. \citet{edunov-etal-2018-classical} surveyed a range of classical objective functions and applied them to neural sequence to sequence models. \citet{ma2018bag} proposed to optimize NMT by the bag-of-words training objective. \citet{shao2018greedy} introduced probabilistic n-gram matching to transform the discrete sequence-level objective into the differentiable form.

As shown above, sequence-level training has attracted much attention of researchers and has been deeply studied on autoregressive models. However, though sequence-level training is more essential on non-autoregressive models, its application on NAT has not been well studied before.

\section{Experiments}
\subsection{Setup}
\vspace{5pt}
\noindent{}\textbf{Datasets.} We evaluate our proposed methods on four translation tasks: WMT14 English$\leftrightarrow$German (WMT14 En$\leftrightarrow$De) and WMT16 English$\leftrightarrow$Romanian (WMT16 En$\leftrightarrow$Ro). We use the standard tokenized BLEU ~\cite{papineni2002bleu} to evaluate the translation quality. For WMT14 En$\leftrightarrow$De, we use the WMT 2016 corpus consisting of 4.5M sentence pairs for the training. The validation set is \texttt{newstest2013} and the test set is \texttt{newstest2014}. We learn a joint BPE model \cite{sennrich-etal-2016-neural} with 32K operations to process the data and share the vocabulary for source and target languages. For WMT16 En$\leftrightarrow$Ro, we use the WMT 2016 corpus consisting of 610K sentence pairs for the training. We take \texttt{newsdev-2016} and \texttt{newstest-2016} as development and test sets. We learn a joint BPE model with 32K operations to process the data and share the vocabulary for source and target languages. 

\vspace{5pt}
\noindent{}\textbf{Knowledge Distillation.} Knowledge distillation \cite{hinton2015distilling,kim-rush-2016-sequence} is proved to be crucial for successfully training NAT models. For all the translation tasks, we first train an autoregressive model as the teacher and apply sequence-level knowledge distillation to construct the distillation corpus where the target side of the training corpus is replaced by the output of the autoregressive Transformer model. We use the distillation corpora to train NAT models. 

\vspace{5pt}
\noindent{}\textbf{Baselines.} We take the base version of Transformer \cite{vaswani2017attention} as our autoregressive baseline as well as the teacher model. The NAT baseline takes the same structure as the base Transformer except that we modify the attention mask of the decoder, so it does not mask out the future tokens. 
We perform uniform copy from source embeddings \cite{gu2017non} to construct decoder inputs. We use a target length predictor to predict the length of the target sentence, which takes the encoder hidden states as inputs and feeds it to a softmax classifier after an affine transformation. We use the golden length during the training and the predicted length during the inference. 

\vspace{5pt}
\noindent{}\textbf{Rescoring.} For inference, we follow the common practice of noisy parallel decoding \cite{gu2017non}, which generates a number of decoding candidates in parallel and selects the best translation via rescoring with the autoregressive teacher. We generate multiple translation candidates by predicting the target length $T$ and generate translations with lengths ranging from $[T-B,T+B]$, where $B$ is the beam size. The autoregressive teacher calculates the cross-entropy loss of the $n=2B+1$ translations and selects the translation with the lowest loss.

\vspace{5pt}
\noindent{}\textbf{Hyperparameters.} In the main experiments, we set the topk size $k$ to $5$, the sampling times $n$ to $10$ and the $N$ in Bag-of-Ngrams to 2. We use the ROUGE-2 score as the reward in reinforcement learning. In the pretraining stage of the three-stage training, the number of training steps is $300k$ for WMT14 En$\leftrightarrow$De and $150k$ for WMT16 En$\leftrightarrow$Ro. In the second stage, we use the BoN objective to finetune the model for $3k$ steps. In the final stage, we use sequence-level evaulation metrics to finetune the model for $300$ steps. The batch size is $128k$ for pretraining and $512k$ for finetuning. For WMT14 En$\leftrightarrow$De, we use a dropout rate of $0.3$ during the pretraining and a dropout rate of $0.1$ during the finetuneing. For WMT16 En$\leftrightarrow$Ro, we use a dropout rate of $0.3$ during the pretraining and finetuneing. We also use $0.01$ $L_2$ weight decay and label smoothing with $\epsilon=0.1$ for regularization. We follow the weight initialization schema from BERT \cite{devlin-etal-2019-bert}. All models are optimized with Adam \cite{DBLP:journals/corr/KingmaB14} with $\beta=(0.9,0.98)$ and $\epsilon=10^{-8}$. The learning rate warms up to $5\cdot10^{-4}$ within $10k$ steps, and then decays with the inverse square-root schedule. We use $8$ GeForce RTX $3090$ GPUs for the training.

\subsection{Main Results}
\begin{table}[t]
\centering
\caption{The performance (test set BLEU) of our methods on all of our benchmarks. All models except Transformer are purely non-autoregressive, using a single forward pass during the argmax decoding. NAT-Base is the baseline NAT model. All other NAT models are finetuned from the NAT-Base.}
\begin{tabular}{clccccc}
\toprule
&\multirow{2}{*}{\textbf{Model}} & \multirow{2}{*}{\textbf{Speedup}} & \multicolumn{2}{c}{\textbf{WMT14}} & \multicolumn{2}{c}{\textbf{WMT16}} \\ 
&&& \textbf{EN-DE} & \textbf{DE-EN} & \textbf{EN-RO} & \textbf{RO-EN}  \\\midrule 
\multirow{2}{*}{Base}
&Transformer & 1.0$\times$&  27.42 & 31.63 & 34.18 &  33.72 \\
&NAT-Base & 15.6$\times$& 19.51 & 24.47 & 28.89 & 29.35 \\
\midrule
\multirow{4}{*}{RL}
&Reinforce-Base & 15.6$\times$&23.23 & 27.59 & 29.67 & 29.85  \\ 
&Reinforce-Step & 15.6$\times$& 23.76 & 28.08 & 30.14 &  30.31 \\ 
&Reinforce-Topk & 15.6$\times$& 24.68 & 29.05 & 30.73 & 30.88  \\ 
&Traverse-Ref & 15.6$\times$& 25.15 & 29.50 & 31.12 &  31.34 \\ 
\midrule
\multirow{4}{*}{BoN}
&BoW-$Cos$&15.6$\times$& 23.90 & 28.11 & 29.72 & 29.69  \\ 
&BoW-$L_2$&15.6$\times$& 24.22 & 29.03 & 30.08 & 29.93 \\ 
&BoW-$L_1$&15.6$\times$& 24.75& 29.41 & 31.01 & 31.19  \\ 
&BoN-$L_1$ (N=2)& 15.6$\times$& 25.28 & 29.66 & 31.37 & 31.51 \\ 
\midrule
3-Stage&RL+BoN& 15.6$\times$& \textbf{25.54} & \textbf{29.91} & \textbf{31.69} &  \textbf{31.78} \\ 
\bottomrule
\end{tabular}
\label{tab:our_results}
\end{table}
We first compare the performance of our proposed methods, including the Reinforcement Learning (RL) based methods (i.e., Reinforce-Base, Reinforce-Step, Reinforce-Topk and Traverse-Ref) and the Bag-of-Ngrams (BoN) based methods (i.e., BoW-$Cos$, BoW-$L_2$, BoW-$L_1$ and BoN-$L_1$ (N=2)). We also adopt the three-stage training strategy to combine the best performing methods of the above two categories (i.e., Traverse-Ref and BoN-$L_1$ (N=2)), which is denoted as BoN+RL. Table \ref{tab:our_results} reports the experiment results of our methods, from which we have the following observations.

\vspace{5pt}
\noindent{}1. {Sequence-level training can effectively improve the performance of non-autoregressive models.} All the methods listed in Table \ref{tab:our_results} can greatly improve the translation quality of non-autoregressive models. Even the simplest method Reinforce-Base achieves more than 3 BLEU improvements on the WMT14 dataset, indicating that sequence-level training is very suitable for non-autoregressive models.

\vspace{5pt}
\noindent{}2. {The methods we propose for variance reduction are helpful to enhance the performance of the reinforcement learning.} Comparing the reinforcement learning based methods, Reinforce-Step reduces the estimation variance by replacing the sentence reward with step reward, which improves Reinforce-Base by about 0.5 BLEU. Reinforce-Topk further improves Reinforce-Step by about 0.8 BLEU by eliminating the variance of important words. Finally, Traverse-Ref gives a method to traverse the whole search space for Reference-Based rewards, which improves Reinforce-Topk by about 0.4 BLEU. In summary, the methods we propose for variance reduction are helpful to enhance the performance of reinforcement learning.

\vspace{5pt}
\noindent{}3. {Among the three BoW training objectives, the $L_1$ distance is very suitable for the training of non-autoregressive models.} Comparing the three Bag-of-Words objectives, BoW-$L_1$ achieves the best performance and largely outperforms the other two objectives, indicating that the $L_1$ distance of BoW is very suitable for the training of non-autoregressive models. Regarding the Bag-of-Ngrams objective, the main limitation is that many distance metrics like $L_2$ distance and cosine distance cannot be calculated, and the observation on BoW can alleviate this concern to some extent.

\vspace{5pt}
\noindent{}4. {Three-stage training can effectively combine reinforcement learning and bag-of-ngrams.} Three-stage training achieves the best performance by combining the best methods of the two categories (i.e., Traverse-Ref and BoN-$L_1$ (N=2)), which improves the NAT baseline by more than 5 BLEU scores on the WMT14 dataset and more than 2 BLEU scores on the WMT16 dataset. We use Seq-NAT to represent this method.

\subsection{Sequence-Level Training for Iterative NAT}
In the previous section, we have verified the effect of sequence-level training on the vanilla NAT, which is non-iterative and uses a single forward pass during the decoding. In this section, we conduct experiments to evaluate the effect of sequence-level training on iterative NAT, which is an important class of NAT models. We use the Conditional Masked Language Model (CMLM) with mask-predict decoding \cite{ghazvininejad2019maskpredict} as our baseline model, which is a strong iterative NAT model. We apply sequence-level training to finetune the CMLM baseline and call this method Seq-CMLM. Figure \ref{fig:seqcmlm} shows the BLEU scores of CMLM and Seq-CMLM under different number of iterations.

From Figure \ref{fig:seqcmlm}, we can see that Seq-CMLM consistently outperforms CMLM on all number of iterations. Even with 10 iterations, Seq-CMLM can achieve an improvement of 0.42 BLEU to CMLM, reaching a BLEU score of 27.36, showing that sequence-level training is also very effective on Iterative NAT. 

\begin{figure}[ht]
  \begin{center}
    \includegraphics[width=0.8\columnwidth]{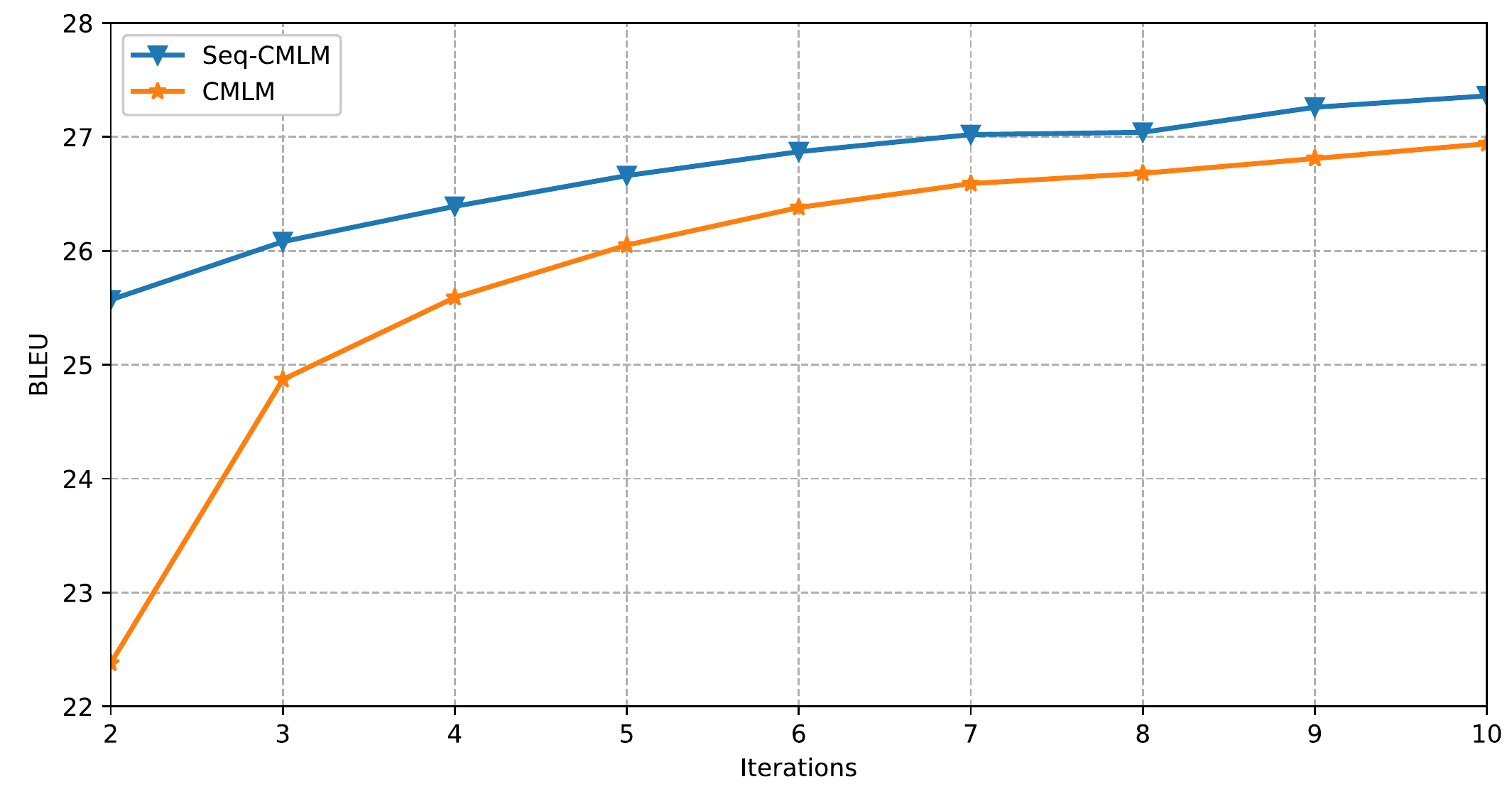}
    \caption{BLEU scores of CMLM and Seq-CMLM on the test set of WMT14 En-De under different number of iterations. For the elegance of the figure, we removed the 1 iteration result, which is 17.47 for CMLM and 24.90 for Seq-CMLM.}
    \label{fig:seqcmlm}
  \end{center}
\end{figure}

\subsection{Speedup in Batch Decoding}
Non-autoregressive models have high speedup in sentence by sentence translation, but this advantage will gradually decrease when we increase the size of decoding batch, making the advantage of NAT in practical application questioned. We resolve this concern by measuring the translation latency of NAT and AT models under different sizes of decoding batch. We conduct experiments on the test set of WMT14 En$\rightarrow$De and show the results in Figure \ref{fig:speedbatch}.
\begin{figure}[ht]
  \begin{center}
    \includegraphics[width=0.8\columnwidth]{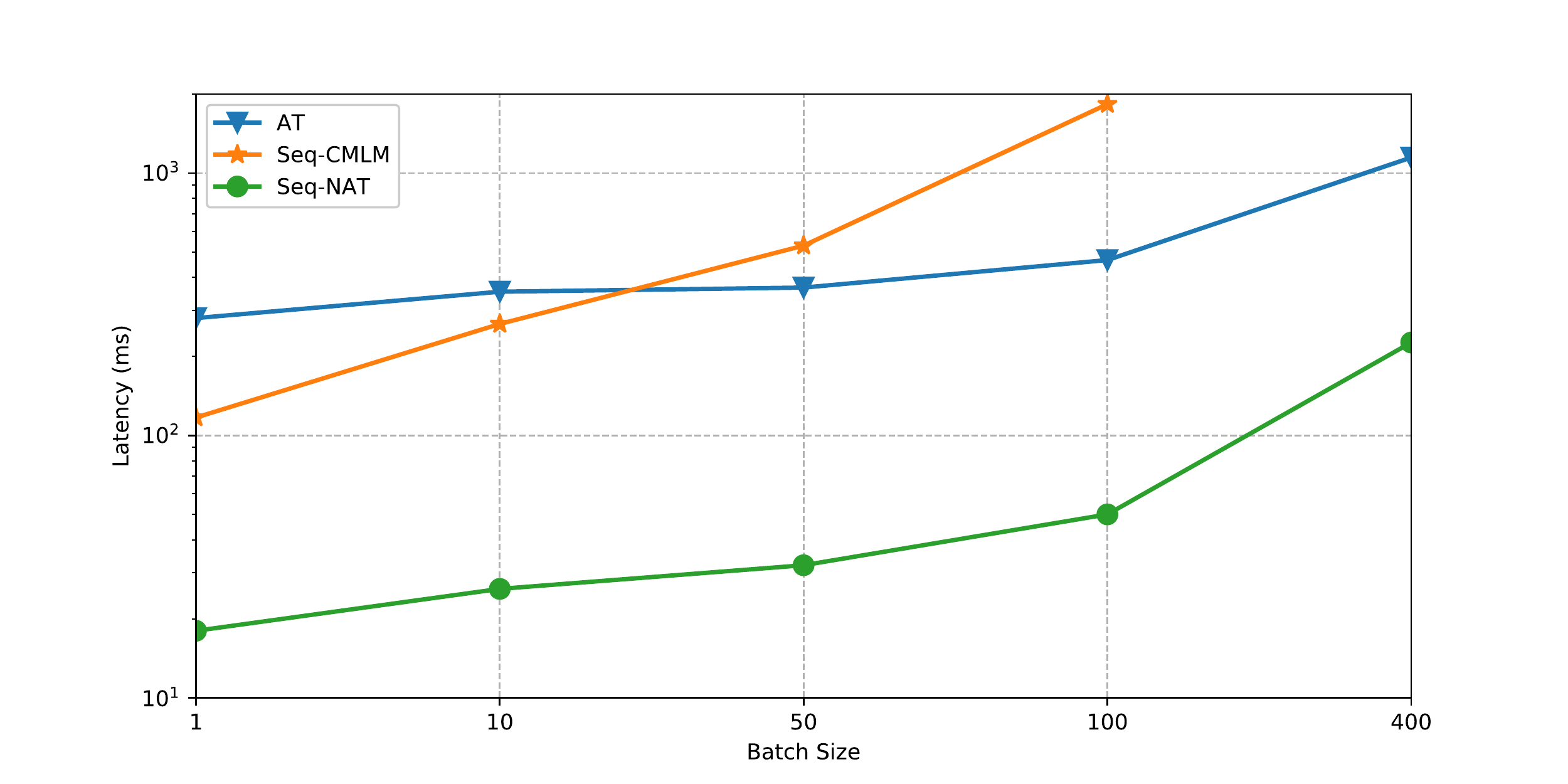}
    \caption{The translation latency of AT, Seq-NAT and Seq-CMLM measured on the test set of WMT14 En$\rightarrow$De. The decoding batch of 400 sentences is not applicable to Seq-CMLM due to the memory limit.}
    \label{fig:speedbatch}
  \end{center}
\end{figure}

From Figure \ref{fig:speedbatch}, as the size of decoding batch increases, both NAT models have higher translation latency. Notably, the iterative model Seq-CMLM becomes even much slower than the autoregressive model when using large batch size. On the contrary, the one-iteration model Seq-NAT still maintains more than 5$\times$ speedup during the batch decoding, demonstrating the efficiency of non-autoregressive generation.

\subsection{Correlation with Translation Quality}
\label{sec:corr}

In this section, we conduct experiments to analyze the correlation between loss functions and the translation quality. We are interested in how the cross-entropy loss and BoN objective correlate with the translation quality. We do not analyze the reinforcement learning based methods since they do not calculate the loss function, but directly estimate the gradient of the loss. We use the GLEU score to represent the translation quality, which is more accurate than BLEU in sentence-level evaluation \cite{wu2016google}. We conduct experiments on the validation set of WMT14 En$\rightarrow$De, which contains 3,000 sentences. Firstly, we load the NAT-Base model and calculate the loss of every sentence in the validation set. Then we use the NAT-Base model to decode the validation set and calculate the GLEU score of every sentence. Finally, we calculate the pearson correlation between the 3,000 GLEU scores and losses.

For the cross-entropy loss, we normalize it by the target sentence length. The BoN training objective is the $L_1$ distance normalized by $2(T-n+1)$. We respectively set $n$ to $2$, $3$ and $4$ to test different n-gram sizes. Table \ref{tab:n} lists the correlation results.

\begin{table}[ht]
\centering
\caption{The pearson correlation bewteen loss functions and translation quality. $n=k$ represents the bag-of-kgrams training objective. CE represents the cross-entropy loss.}
\begin{tabular}{c|c|c|c|c}
\toprule
Loss function& CE&$n=2$ &$n=3$ &$n=4$\\
\hline
Correlation&0.56&0.87&0.84&0.79\\
\bottomrule
\end{tabular}
\label{tab:n}
\end{table}

From Table \ref{tab:n}, we can see that all the three BoN objectives outperform the cross-entropy loss by large margins, and the $n=2$ setting achieves the highest correlation 0.87. To find out where the improvements come from, we analyze the effect of sentence length in the following experiment. We evenly divide the dataset into two parts according to the source length. The first part consists of 1500 short sentences and the second part consists of 1500 long sentences. We respectively measure the pearson correlation on the two parts and report the results in Table \ref{tab:joint}:
\begin{table}[ht]
\centering
\caption{The pearson correlation bewteen loss functions and translation quality on short sentences and long sentences.}
\begin{tabular}{c|c|c|c}
\toprule
& all& short& long\\
\hline
Cross-Entropy&0.56&0.68&0.44\\
\hline
BoN ($n$=2)&0.87&0.89&0.86\\
\bottomrule
\end{tabular}
\label{tab:joint}
\end{table}

From Table \ref{tab:joint}, we can see that the correlation of the cross-entropy loss drops as the sentence length increases, where the BoN objective still has a strong correlation on long sentences. The reason is not difficult to explain. The cross-entropy loss requires the strict alignment between the translation and reference. As the sentence length grows, it becomes harder for NAT to align the translation with the reference, which leads to a decrease of correlation between cross-entropy loss and translation quality. In contrast, the BoN objective is robust to unaligned situations, so its correlation with translation quality stays strong when translating long sentences.

\subsection{Effect of Training Strategy}
In this section, we analyze the effect of training strategy that combines the word-level loss and the two methods based on reinforcement learning and bag-of-ngrams. Before discussing the training strategy, we first give the training speed of each method in Table \ref{tab:speed}. As we can see, Traverse-Ref is the slowest method, which is nearly 10 times slower than BoN. Therefore, when choosing a training strategy, it is necessary to avoid a large number of calculations of Traverse-Ref.

\begin{table}[ht]
\centering
\caption{The training time required for different methods to process 64k tokens. The time is measured on the training set of WMT14 En-De with a single GeForce RTX 3090 GPU. CE is the cross-entropy loss. RF is the abbreviation of Reinforce.}
\begin{tabular}{c|c|c|c|c|c|c|c}
\toprule
Method&CE& BoW& BoN ($n$=2)& RF-Base& RF-Step& RF-Topk&Traverse-Ref\\
\hline
Time&1.2s&1.5s&7.1s&2.2s&16.9s&31.3s&63.2s\\
\bottomrule
\end{tabular}
\label{tab:speed}
\end{table}

We consider four training strategies that involve the word-level cross-entropy loss, the Traverse-Ref loss and the BoN loss. Firstly, we consider the two-stage strategy that uses the cross-entropy loss for pretraining and finetunes the model by the weighted summation of the Traverse-Ref and BoN losses. The second strategy follows the joint training strategy in \citet{DBLP:conf/aaai/ShaoZFMZ20} to combine the BoN and cross-entropy loss for pretraining, and then finetune the model sequentially by BoN and Traverse-Ref. The latter two strategies adopt the three-stage strategy that uses the cross-entropy loss for pretraining and sequentially uses Traverse-Ref and BoN for finetuning. We report the BLEU scores of the four strategies together with the training time in Table \ref{tab:strategy}. 

\begin{table}[ht]
\centering
\caption{Validation BLEU scores and training time of different training strategies on WMT14 En-De. CE represents the cross-entropy loss and TR represents the Traverse-Ref loss. The training time is measured on 8 GeForce RTX 3090 GPUs.}
\begin{tabular}{l|c|c}
\toprule
Strategy&BLEU&Time\\
\hline
1. CE---BoN+TR&24.56&65.6h\\
\hline
2. CE+BoN---BoN---TR&24.82&93.1h\\
\hline
3. CE---TR---BoN&24.63&63.3h\\
\hline
4. CE---BoN---TR&24.69&37.5h\\
\bottomrule
\end{tabular}
\label{tab:strategy}
\end{table}

Table \ref{tab:strategy} shows that the second strategy achieves the best performance but suffers from high training cost. The fourth strategy is more economical, which achieves a slightly lower BLEU but greatly shortens the training time. Compared with the other two strategies, it outperforms them on both BLEU and training time. Therefore, we finally adopt the fourth strategy to combine the word-level training and sequence-level training methods.

\subsection{Effect of Hyperparameters}
In this section, we analyze the effect of some hyperparameters in our method that will affect the model performance, including the topk size $k$ and the reward fuction in reinforcement learning, the ngram size $n$ in bag-of-ngrams training and the batch size for finetuning.

\vspace{5pt}
\noindent{}\textbf{Topk Size.} Reinforce-Topk is proposed to reduce the estimation variance by traversing the topk words, which is important in the gradient estimation. Intuitively, a larger $l$ will make the model stronger. When $k$ is 0, Reinforce-Topk degenerates to Reinforce-Step. When $k$ equals to the vocabulary size $|V|$, Reinforce-Topk has the same performance with Traverse-Ref. However, using such a large $k$ will make the training very slow. Therefore, we need to find an appropriate $k$ to balance the performance and training cost. We conduct experiments on the validation set of WMT14 En-De to see the effect of topk size $k$ and illustrate our results in Figure \ref{fig:topk}.

\begin{figure}[ht]
  \begin{center}
    \includegraphics[width=0.8\columnwidth]{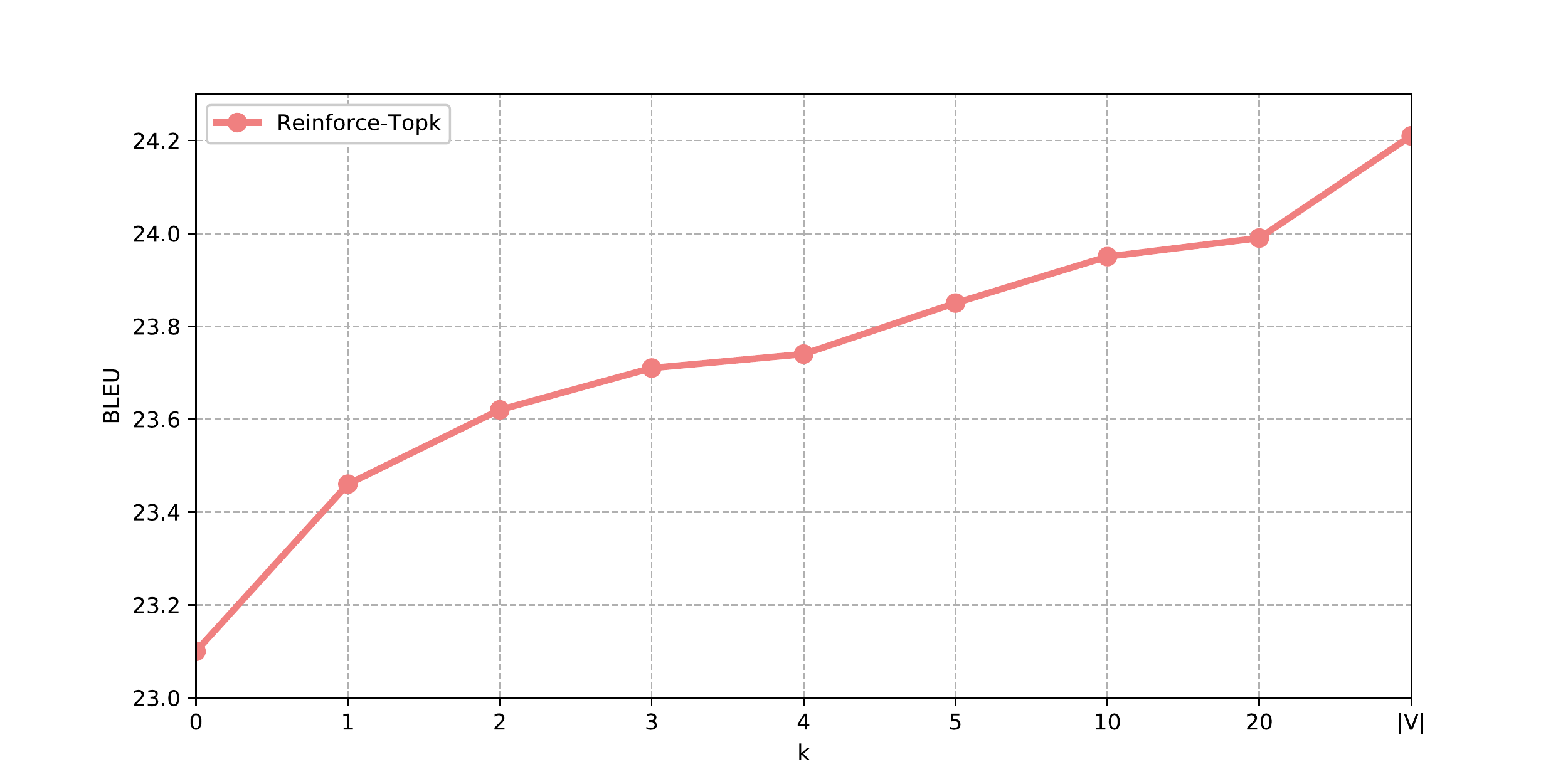}
    \caption{BLEU scores of Reinforce-Topk with different $k$ on the validation set of WMT14 En-De.}
    \label{fig:topk}
  \end{center}
\end{figure}

From Figure \ref{fig:topk}, we can see that the model performance steadily improves as $k$ rises from 0 to 5. When $k$ rises from 5 to 10, the model performance is also slightly improved. However, we can barely see improvements from $k=10$ to $k=20$, showing that the appropriate $k$ is between 5 to 10. In addition, we use Traverse-Ref to show the $k=|V|$ result, which achieves considerable improvements to Reinforce-Topk.

\vspace{5pt}
\noindent{}\textbf{Reward Function.} The performance of reinforcement learning based methods is influenced by the reward function it uses. Our methods have almost no restriction on the reward function, where only Traverse-Ref requires the reward function to Reference-Based. Therefore, we choose three widely used Reference-Based rewards BLEU \cite{papineni2002bleu}, GLEU \cite{wu2016google} and ROUGE-2 \cite{lin-2004-rouge} as candidates. We use the three rewards to finetune the NAT baseline and report their results in Table \ref{tab:reward}. We also directly evaluate the three rewards by the pearson correlation coefficient with translation quality. We use the WMT16 DAseg De-En dataset for evaluation, which consists of 560 source sentences, model translations, reference sentences and human scores. We get the rewards of the model translations and calculate the pearson correlation coefficient between rewards and human scores, as shown in Table \ref{tab:reward}. We can see that there is no significant difference in the BLEU performance of these three rewards. In terms of the correlation, BLEU underperforms ROUGE-2 and GLEU by a large margin, which is possibly due to instability of BLEU as there are usually little matches of 3-gram or 4-gram in sentence-level evaluation. We finally use the ROUGE-2 as the reward function because of its overall performance and fast calculation in our implementation.

\begin{table}[ht]
\centering
\caption{BLEU scores on the validation set of WMT14 En-De when using BLEU, GLEU or ROUGE-2 as reward functions, and the pearson correlation coefficient of these rewards.}
\begin{tabular}{c|c|c|c|c|c}
\toprule
&Correlation&Reinforce-Base&Reinforce-Step&Traverse-Ref&Average\\
\hline
BLEU&0.389&22.58&23.01&24.12&23.24\\
GLEU&0.482&22.56&23.17&24.16&23.30\\
ROUGE-2&0.483&22.39&23.14&24.25&23.26\\
\bottomrule
\end{tabular}
\label{tab:reward}
\end{table}

\vspace{5pt}
\noindent{}\textbf{Ngram Size.} Table \ref{tab:our_results} has shown that the bag-of-ngrams (N=2) objective outperform the bag-of-words objective, but the effect of different ngram sizes $n$ has not been analyzed. Therefore, we conduct experiments on the validation set of WMT14 En-De to see the performance of bag-of-ngrams objectives with different choices of $n$, and we also provide the training speed of BoN-$L_1$ with different $n$. Results are listed in Table \ref{tab:ngram}. We can see that $n=2$ slightly outperforms other choices of $n$, which is consistent with the correlation result in Table \ref{tab:n}. Furthermore, BoN-$L_1$ with $n=2$ is much faster than other choices of $n$ during the training, so we set $n=2$ in the main experiement.

\begin{table}[ht]
\centering
\caption{Validation BLEU scores of BoN-$L_1$ with different $n$ on WMT14 En-De and the time required to process 64k tokens during the training. The time is measured with a single GeForce RTX 3090 GPU.}
\begin{tabular}{c|c|c|c}
\toprule
$n$ &$n=2$ &$n=3$ &$n=4$\\
\hline
BLEU&24.37&24.29&24.07\\
\hline
Time&7.1s&9.7s&12.3s\\
\bottomrule
\end{tabular}
\label{tab:ngram}
\end{table}

\vspace{5pt}
\noindent{}\textbf{Batch Size for Finetuning.} In the training of deep neural models, a larger batch size usually leads to stronger performance, which comes with the cost of greater training costs. In the sequence-level training scenario, since we only need to fine-tune the model for a few steps, we can increase the batch size within a reasonable range, which only slightly increases the training cost but brings considerable improvements on the model performance. To show the effect of the batch size, we use different batch sizes during the BoN finetuning and report the corresponding BLEU scores and total training time in Table \ref{tab:batch}. We can see that the BLEU score steadily increases as the batch size for finetuning increases. In terms of training time, even when we use a batch size of 512k, which is 4 times the size of the pretraining, the training time is only 1.25 times of the NAT baseline.

\begin{table}[ht]
\centering
\caption{Validation BLEU scores on WMT14 En-De and training costs when using different batch size during the finetuning. NoFT represents the NAT baseline without finetuning. The training time is measured on $8$ GeForce RTX $3090$ GPUs.}
\begin{tabular}{c|c|c|c|c|c|c}
\toprule
Batch Size &NoFT&32k&64k&128k&256k&512k\\
\hline
BLEU&19.08&23.77&23.93&24.15&24.27&24.37\\
\hline
Time&24.8h&25.2h&25.6h&26.3h&27.9h&31.0h\\
\bottomrule
\end{tabular}
\label{tab:batch}
\end{table}

\subsection{Effect of Sentence Length}
In section \ref{sec:corr}, we analyze the correlation between loss functions and the translation quality under different sentence lengths, which shows that sequence-level losses greatly outperform the word-level loss in terms of correlation when evaluating long sentences. In this section, we calculate the BLEU performance of baseline methods and our model on different sentence lengths and see whether the better correlation contributes to better BLEU performance. We conduct the experiment on the validation of WMT14 En$\rightarrow$De and divide the sentence pairs into different length buckets according to the length of the source sentence. We use Seq-NAT to represent our best performing method, and calculate the BLEU scores of baseline models and Seq-NAT under different length buckets. The results are shown in Figure \ref{fig:length}.

\begin{figure}[ht]
  \begin{center}
    \includegraphics[width=0.8\columnwidth]{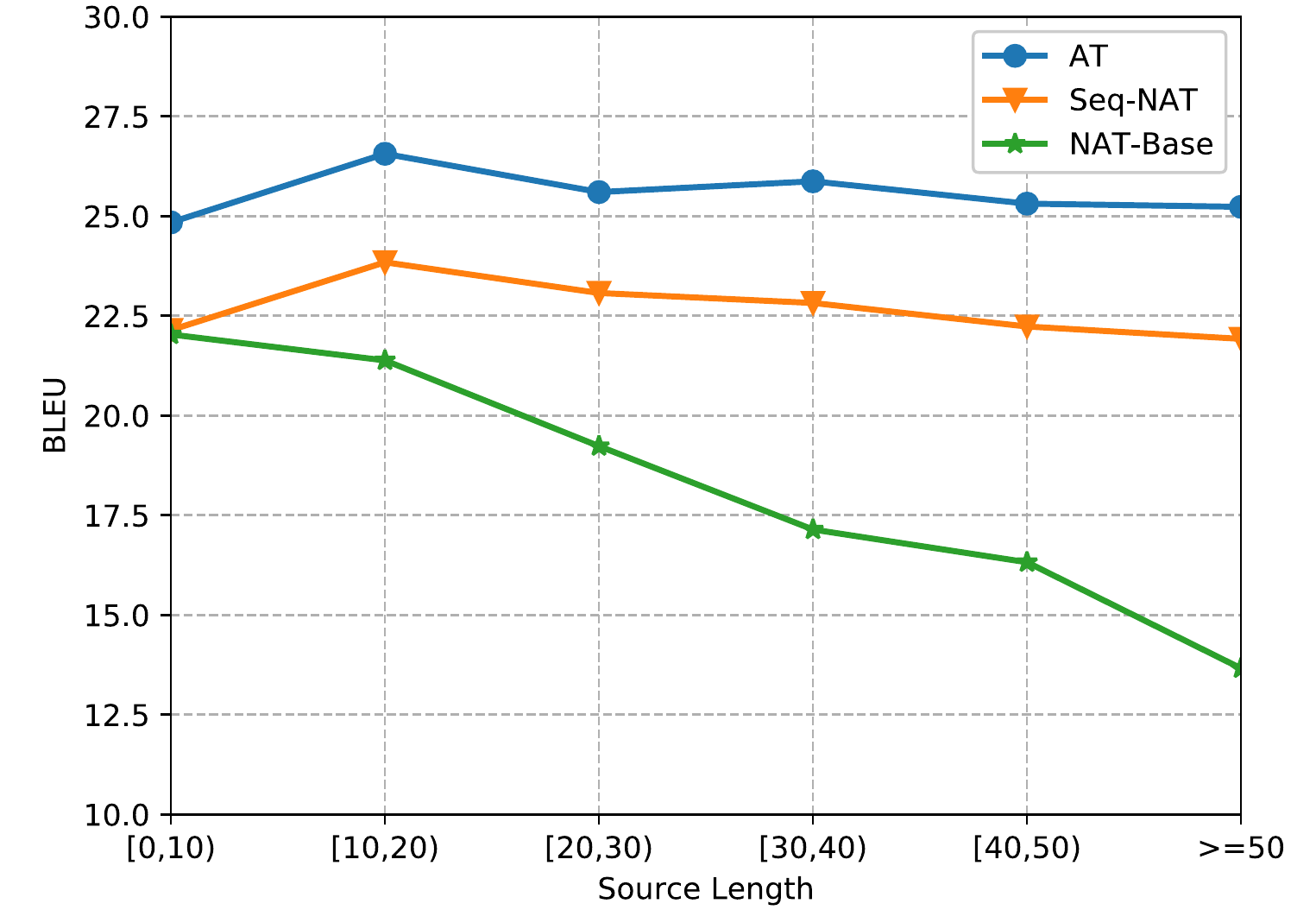}
    \caption{Validation BLEU scores of baseline methods and Seq-NAT on WMT14 En-De under different length buckets.}
    \label{fig:length}
  \end{center}
\end{figure}

From Figure \ref{fig:length}, we can see that NAT-Base and Seq-NAT have similar performance when translating short sentences. However, the translation quality of NAT-Base drops quickly as sentence length increases, where the autoregressive Transformer and Seq-NAT have stable performance over different sentence lengths, which is in good agreement with the correlation results. As the sentence length grows, the correlation between the cross-entropy loss and the translation quality drops, which leads to the weakness of NAT in translating long sentences. On the contrary, sequence-level losses evaluate the translation quality of long sentences with high correlations, so Seq-NAT has stable performance on long sentences.
\subsection{Performance Comparision}
We use Seq-NAT to represent our best performing method. In Table \ref{tab:all_results}, we compare the performance of Seq-NAT against the autoregressive Transformer and strong non-iterative NAT baseline models. Table \ref{tab:all_results} shows that Seq-NAT outperforms most existing NAT systems, and the performance gap between Seq-NAT and the autoregressive teacher is about 2 BLEU on average. Rescoring 9 candidates further improves the translation quality and narrows the performance gap to about 0.8 BLEU on average. It is also worth noting that Seq-NAT does not affect the translation speed, which has the same speedup 15.6$\times$ as NAT-Base. After rescoring 9 candidates, Seq-NAT still maintains 9.0$\times$ speedup.

\begin{table}[t]
\centering
\caption{Performance comparison between our method Seq-NAT and existing methods. The speedup is measured on the WMT14 En-De test set with batch size 1. ``---'' indicates that the result is not reported. $n$ is the number of candidates rescored by the autoregressive teacher.}
\begin{tabular}{lccccc}
\toprule
\multirow{2}{*}{\textbf{Model}}  & \multicolumn{2}{c}{\textbf{WMT14}} & \multicolumn{2}{c}{\textbf{WMT16}}& \multirow{2}{*}{\textbf{Speedup}} \\
& \textbf{EN-DE} & \textbf{DE-EN} & \textbf{EN-RO} & \textbf{RO-EN} \\\midrule
\textbf{Autoregressive} &&&&& \\
\midrule
Transformer &  27.42 & 31.63 & 34.18 &  33.72 &1.0$\times$\\
\midrule
\multicolumn{2}{l}{\textbf{Non-Autoregressive w/o Rescoring}} &&&&  \\
\midrule
NAT-FT \cite{gu2017non} &  17.69 & 21.47 & 27.29 &  29.06&15.6$\times$\\
LT \cite{kaiser2018fast}  &  19.80 & --- & --- &  ---&5.8$\times$\\
CTC \cite{libovicky2018end}  &  17.68 & 19.80 & 19.93 & 24.71 &---\\
ENAT \cite{guo2019non} & 20.65 & 23.02 & 30.08 & --- & 25.3$\times$\\
NAT-REG \cite{wang2019non} &  20.65 &24.77 &--- & --- &27.6$\times$\\
NAT-Hints \cite{li-etal-2019-hint} &  21.11& 25.24&--- &--- &30.8$\times$\\ 
Reinforce-NAT \cite{shao-etal-2019-retrieving} &19.15 &22.52 &27.09 &27.93 &10.77$\times$\\
BoN-Joint+FT \cite{DBLP:conf/aaai/ShaoZFMZ20} &20.90 &24.61 &28.31 &29.29 &10.73$\times$\\
imitate-NAT \cite{wei-etal-2019-imitation} & 22.44 & 25.67 & 28.61 & 28.90 & 18.6$\times$\\
FlowSeq \cite{Ma_2019} & 23.72 & 28.39 & 29.73 & 30.72 &---\\
NART-DCRF \cite{NIPS2019_8566} &  23.44 &  27.22 &  --- & ---&10.4$\times$\\
ReorderNAT \cite{ran2019guiding} & 22.79 & 27.28 & 29.30 & 29.50 & 16.1$\times$\\
PNAT \cite{bao2019nonautoregressive} & 23.05 & 27.18 & --- & ---& 7.3$\times$\\
FCL-NAT \cite{fine} & 21.70 & 25.32 & ---& ---&28.9$\times$\\
AXE \cite{Aligned} & 23.53 & 27.90 & 30.75 & 31.54 & ---\\
EM \cite{pmlr-v119-sun20c} & 24.54 & 27.93 & --- &  --- & 16.4$\times$\\
Imputer \cite{saharia-etal-2020-non} & \textbf{25.80} & 28.40 & \textbf{32.30} & 31.70 & ---\\
Seq-NAT (ours)&   25.54   & \textbf{29.91} & 31.69 & \textbf{31.78}&15.6$\times$\\
\midrule
\multicolumn{2}{l}{\textbf{Non-Autoregressive w/ Rescoring}} &&&&  \\
\midrule
NAT-FT (n=10) \cite{gu2017non}& 18.66 & 22.41 & 29.02 & 30.76 & 7.68$\times$\\
NAT-FT (n=100) \cite{gu2017non}& 19.17 & 23.20 & 29.79& 31.44 & 2.36$\times$\\
LT (n=10) \cite{kaiser2018fast}& 21.0 & --- & --- & --- & --- \\
ENAT (n=9) \cite{guo2019non}& 24.28 & 26.10 & \textbf{34.51} & --- & 12.4$\times$\\
NAT-REG (n=9) \cite{wang2019non}& 24.61 & 28.90 & --- & --- & 15.1$\times$\\
NAT-Hints (n=9) \cite{li-etal-2019-hint}& 25.20 & 29.52 & --- & --- & 17.8$\times$\\
imitate-NAT (n=7) \cite{wei-etal-2019-imitation}& 24.15 & 27.28 & 31.45 & 31.81 & 9.70$\times$\\
FlowSeq (n=15) \cite{Ma_2019}& 25.03 & 30.48 & 31.89 & 32.43 & ---\\
NART-DCRF (n=9) \cite{NIPS2019_8566}& 26.07 & 29.68 & --- & --- & 6.14$\times$\\
PNAT (n=7) \cite{bao2019nonautoregressive}& ---  & 27.90& ---& ---& 3.7$\times$\\
FCL-NAT (n=9) \cite{fine}& 25.75 & 29.50 & --- & ---&16.0$\times$\\
EM (n=9) \cite{pmlr-v119-sun20c}& 25.75 & 29.29 & --- & ---& 9.14$\times$\\
Seq-NAT (n=9, ours) &   \textbf{26.35}   & \textbf{30.70} & 33.21 & \textbf{33.28} &9.0$\times$\\

\bottomrule
\end{tabular}
\label{tab:all_results}
\end{table}
\subsection{Case Study}
In Table \ref{tab:case_study}, we present three translation cases from the validation set of WMT14 De-En to analyze how sequence-level training improves the translation quality of NAT. We can see from the three cases that the NAT baseline suffers from over-translation and under-translation errors especially when translating long sentences. The output of NAT-Base contains many repeated translations like `aggressive', `shadow' and `14'. Additionally, the translation is incomplete since many information are missing. As we mentioned before, this is due to the limitation of the word-level cross-entropy loss we use, which evaluates the generation quality of each position independently and does not model the target-side sequential dependency, making NAT only focus on local correctness and ignore the overall translation quality.

When we look at the translation results of Seq-NAT, we can see that the errors of over-translation and under-translation are significantly reduced. Although there are still a few repeated translations when translating long sentences, the translation results are basically accurate and comparable to the autoregressive Transformer. Compared with the NAT baseline, Seq-NAT focuses more on the overall accuracy after the sequence-level training, which greatly improves the translation quality.

\begin{table}[t]
\centering
\caption{Three translation cases in the validation set of WMT14 De-En. Source and Target are respectively the source sentence and reference sentence. AT is the output of the autoregressive Transformer. NAT-Base is the output of the NAT baseline. Seq-NAT is the output of our model.}
\begin{tabular}{r|l}
\toprule
{Source} & Es gibt Krebsarten , die aggressiv und andere , die indolent sind .
 \\
\hline
{Target} &  There are aggressive cancers and others that are indolent .\\
\hline
{AT} & There are cancers that are aggressive and others that are indolent . \\
\hline
{NAT-Base} & There are cancers cancer aggressive aggressive others are indindent . \\
\hline
{Seq-NAT} & There are cancers that are aggressive and others that indolent . \\
\midrule
\midrule
\multirow{3}{*}{Source} & Wir wissen ohne den Schatten eines Zweifels , dass wir ein echtes \\&neues Teilchen haben , und dass es dem vom Standardmodell \\&vorausgesagten Higgs-Boson stark ähnelt .
 \\
\hline
\multirow{3}{*}{Target} &We know without a shadow of a doubt that it is a new authentic \\&particle , and greatly resembles the Higgs boson predicted by the \\&Standard Model .\\
\hline
\multirow{3}{*}{AT} & We know without the shadow of a doubt that we have a real new \\&particle , and that it is very similar to the Higgs Boson predicted by \\&the standard model .\\
\hline
\multirow{3}{*}{NAT-Base} & We know without without shadow shadow of doubt doubt that we \\&have a new particle le that it is very similar similar to HiggsgsBoson\\& predicted by the standard model . \\
\hline
\multirow{3}{*}{Seq-NAT} & We know without the shadow of a doubt that we have a real new \\&particle and that it is very similar to the Higgs-Boson predicted by \\&the standard model . \\
\midrule
\midrule
\multirow{4}{*}{Source} & und noch tragischer ist , dass es Oxford war - eine Universität, die\\& nicht nur 14 Tory-Premierminister hervorbrachte , sondern sich bis\\&  heute hinter einem unverdienten Ruf von Gleichberechtigung und\\& Gedankenfreiheit versteckt .\\
\hline
\multirow{3}{*}{Target} & even more tragic is that it was Oxford , which not only produced 14 \\&Tory prime ministers , but , to this day , hides behind an ill-deserved \\&reputation for equality and freedom of thought . \\
\hline
\multirow{3}{*}{AT} & And more tragically , it was Oxford - a university that not only \\&produced 14 Tory prime ministers but hides to this day behind an \\&undeserved reputation of equality and freedom of thought .\\ 
\hline
\multirow{4}{*}{NAT-Base} & More more tragic tragic it was Oxford Oxford Oxford university \\&university university not only 14 14 14 Torprime prime ministers but \\&but hihihihiday behind an unved call for equality and freedom of \\&thought .
\\
\hline
\multirow{3}{*}{Seq-NAT} & More is even more tragic that it was Oxford - a university that not \\&only produced 14 Tory prime ministers , but continues continues far\\& hidden behind an unved call for equality and freedom of thought .\\
\bottomrule
\end{tabular}
\label{tab:case_study}
\end{table}

\section{Conclusion}
Non-autoregressive translation achieves significant decoding speedup through generating target words independently and simultaneously. However, the word-level cross-entropy loss cannot evaluate the output of NAT properly. As a result, NAT has a relatively low translation quality and tends to generate translations with over-translation and under-translation errors. In this article, we propose to train NAT with sequence-level training objectives. Firstly, we propose to train NAT to optimize the sequence-level evaluation metric based on novel reinforcement algorithms customized for NAT. Then we introduce a novel Bag-of-Ngrams objective for NAT, which is differentiable and can be calculated efficiently. Finally, we use a three-stage training strategy to combine the strengths of the two training methods and the word-level loss. Experimental results show that our method achieves remarkable performance on all translation tasks.

\starttwocolumn
\bibliography{compling_style}
\end{document}